\documentclass[unnumsec,webpdf,contemporary,large]{oup-authoring-template}%

\graphicspath{{Fig/}}

\theoremstyle{thmstyleone}%
\newtheorem{theorem}{Theorem}
\newtheorem{proposition}[theorem]{Proposition}%
\theoremstyle{thmstyletwo}%
\theoremstyle{thmstylethree}%
\newtheorem{definition}{Definition}

\usepackage{graphicx}
\usepackage{subcaption}
\usepackage{amsmath} 
\usepackage{framed}

\usepackage{color}
\usepackage{booktabs}
\newcommand{\chen}[1]{\textcolor{black}{#1}}
\newcommand{\chenn}[1]{\textcolor{black}{#1}}
\newcommand{\chennn}[1]{\textcolor{black}{#1}}

\begin{document}

\journaltitle{Journal Title Here}
\DOI{DOI HERE}
\copyrightyear{2022}
\pubyear{2019}
\access{Advance Access Publication Date: Day Month Year}
\appnotes{Paper}

\firstpage{1}


\title[DGIB4SL]{Interpretable High-order Knowledge Graph Neural Network for Predicting Synthetic Lethality in Human Cancers}

\author[1]{Xuexin Chen}
\author[1,2,$\ast$]{Ruichu Cai}
\author[1]{Zhengting Huang}
\author[3]{Zijian Li}
\author[4,5]{Jie Zheng}
\author[6,$\ast$]{Min Wu}

\authormark{Chen et al.}

\address[1]{\orgdiv{School of Computer Science}, \orgname{Guangdong University of Technology}, \orgaddress{ \postcode{510006}, \state{Guangdong}, \country{China}}}
\address[2]{\orgdiv{Pazhou Laboratory (Huangpu)}, \orgaddress{\state{Guangzhou}, \country{China}}}
\address[3]{\orgname{Mohamed bin Zayed University of Artificial Intelligence}, \orgaddress{\country{Abu Dhabi}}}
\address[4]{\orgdiv{School of Information Science and Technology}, \orgname{ShanghaiTech University}, \orgaddress{\postcode{201210}, \country{China}}}
\address[5]{\orgdiv{Shanghai Engineering Research Center of Intelligent Vision and Imaging}, \orgname{ShanghaiTech University}, \orgaddress{\postcode{201210}, \country{China}}}
\address[6]{\orgdiv{Institute for Infocomm Research (I$^{2}$R)}, \orgname{A*STAR}, \orgaddress{\postcode{138632},  \country{Singapore}}}

\corresp[$\ast$]{Corresponding authors: \href{cairuichu@gmail.com}{cairuichu@gmail.com}, \href{wumin@i2r.a-star.edu.sg}{wumin@i2r.a-star.edu.sg}}

\received{Date}{0}{Year}
\revised{Date}{0}{Year}
\accepted{Date}{0}{Year}

\abstract{Synthetic lethality (SL) is a promising gene interaction for cancer therapy. Recent SL prediction methods integrate knowledge graphs (KGs) into graph neural networks (GNNs) and employ attention mechanisms to extract local subgraphs as explanations for target gene pairs. However, attention mechanisms often lack fidelity, typically generate a single explanation per gene pair, and fail to ensure trustworthy high-order structures in their explanations.
To overcome these limitations, we propose Diverse Graph Information Bottleneck for Synthetic Lethality (DGIB4SL), a KG-based GNN that generates multiple faithful explanations for the same gene pair and effectively encodes high-order structures. Specifically, we introduce a novel DGIB objective, integrating a Determinant Point Process (DPP) constraint into the standard IB objective, and employ 13 motif-based adjacency matrices to capture high-order structures in gene representations. Experimental results show that  DGIB4SL outperforms state-of-the-art baselines and provides multiple explanations for SL prediction, revealing diverse biological mechanisms  underlying SL inference.
}
\keywords{synthetic lethality, machine learning explainability, graph neural network, information bottleneck}


\maketitle

\section{Introduction}
\begin{figure}[!t]
\centering
\includegraphics[width=0.9\columnwidth]{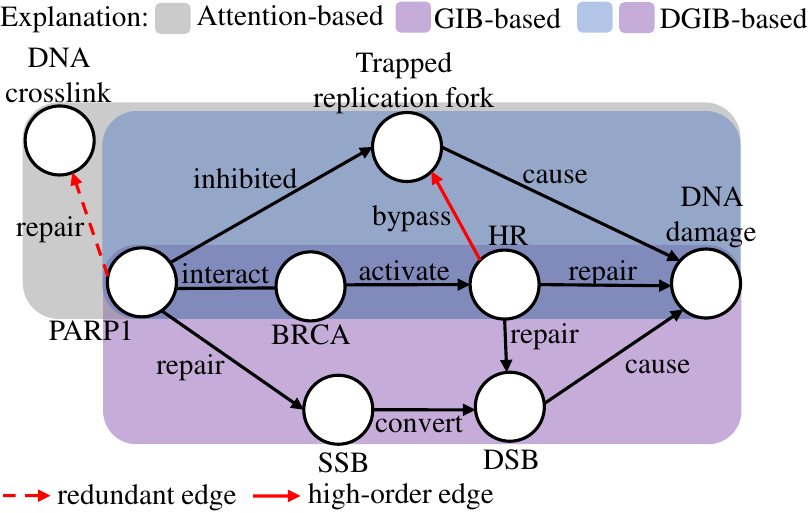}
\caption{Toy example of a knowledge graph with self-loops integrating biological context and relevant mechanisms between the given gene pair BRCA1 and PARP1. The purple and blue subgraphs illustrate mechanisms \chen{where either the conversion of SSBs to DSBs or the blockage of replication forks leads to DNA damage in the absence of HR~\cite{helleday2011underlying}}. The gray subgraph represents the predicted core subgraph of an attention-based method. A GIB-based method  identifies  only one correct subgraph, while our DGIB4SL can find all correct subgraphs (purple and blue). HR, SSB and DSB are abbreviations for ``Homologous Recombination'', ``Single Srand Break'' and ``Double Srand Break'', respectively. The self-loops are not depicted for brevity.}
\label{fig:introduction}
\end{figure}
Synthetic Lethality (SL) is a promising type of genetic interaction where the co-occurrence of two (or more) genetic events leads to cell death, while the occurrence of either event is compatible with cell viability.  
SL has become a cornerstone of anti-cancer drug research, by targeting a gene that is nonessential in normal cells but synthetic lethal with a gene with cancer-specific alterations, which would enable the selective killing of cancer cells without harming normal cells. 
For example, AZD1775, a WEE1 Inhibitor, is based on the SL interaction between WEE1 and p53 mutations~\cite{leijen2016phase}. 
Despite extensive research on SL through high-throughput wet-lab screening methods, these methods often face various challenges, such as high costs and inconsistencies across platforms. 
Thus, predicting SL using computational models becomes highly complementary to wet-lab approaches.

SL prediction approaches can be broadly categorized into statistical inference methods, network-based methods, and supervised machine learning methods. Among these, graph neural networks (GNNs) are currently the most popular model, largely owing to their ability to model complex gene interactions~\cite{wang2021kg4sl}. Although many SL gene pairs have been identified, few of them  have been applied to cancer treatment, as understanding the underlying biological mechanisms remains a critical challenge. 
Unfortunately, most GNNs lack the capability to explain SL mechanisms. 
To address this, methods incorporating  attention mechanisms and knowledge graphs (KGs, a heterogeneous graph containing biological entities and their relationships) have emerged~\cite{wang2021kg4sl,liu2022pilsl,zhang2023kr4sl,zhu2023slgnn}. These approaches enable the identification of crucial edges or semantic features in KGs while predicting SL interaction.

Although KG-based methods with attention mechanisms  improve the interpretability of SL predictions, they still face  three major challenges. 
First, explanations based on the attention mechanisms often lack reliability, since they  tend to assign higher weights to frequent edges and produce unstable explanations across independent runs of the same model~\cite{DBLP:conf/acl/SerranoS19,DBLP:conf/emnlp/WiegreffeP19,DBLP:conf/iclr/BrunnerLPRCW20,DBLP:conf/lrec/GrimsleyMB20,DBLP:journals/tkde/LiSCZYX24}.
As illustrated by the examples in Fig.~\ref{fig:introduction}, the gray subgraph, predicted by attention-based methods, includes a red dashed edge labeled  \chen{``repair''}. This edge, irrelevant to the SL mechanism, \chen{is assigned higher importance due to its frequent occurrence in the KG.}
Second, existing KG-based methods generate only a single core subgraph  to explain  predictions for a given gene pair, even though \chen{multiple subgraphs may provide valid  explanations}~\cite{helleday2011underlying}. As illustrated in Fig.~\ref{fig:introduction}, \chen{the purple subgraph highlights a mechanism where SSB converts to DSB, while the blue subgraph represents replication fork blocking. Both subgraphs explain the SL interaction  between PARP1 and BRCA~\cite{helleday2011underlying}.}
Third, the high-order structures contained in the explanations generated by KG-based methods are often untrustworthy, since the key step of these self-explainable methods, learning gene representation for prediction, cannot capture the information of the interactions between the neighbors (high-order), although the information between a gene and its neighbors can be effectively captured (low-order). 
For instance, as shown in Fig.~\ref{fig:introduction}, the \chen{``DAN damage''} node representation produced by KG-based methods remains unchanged,  regardless of the high-order edge \chen{``HR $\xrightarrow{\text{bypass}}$ Trapped replication fork''}. 
We thus ask: for a gene pair, how to find multiple rather than one faithful core subgraphs and encode their high-order graph information  for prediction?

Our main contribution lies in addressing this question by proposing the  Diverse Graph Information Bottleneck  for Synthetic Lethality (DGIB4SL),  an interpretable GNN model for SL prediction on KGs, \chenn{hinging on the motif-based GNN encoder and our proposed DGIB objective.} First, \chenn{to alleviate instability and the bias toward frequent edges in attention weights, unlike the cross-entropy loss commonly used in attention-based methods, DGIB4SL employs the GIB principle~\cite{DBLP:conf/nips/WuRLL20}}, widely applied in interpretable GNNs~\cite{sun2022graph,miao2022interpretable}, to define a core subgraph from the neighborhood of a gene pair. GIB provides a principled objective function for graph representation learning, determining which data aspects to preserve and  discard~\cite{DBLP:conf/aaai/Pan00021}. However, the standard GIB objective \chenn{identifies only} a single core subgraph for each gene pair. To capture all relevant core subgraphs from the enclosing graph, such as the purple and blue subgraphs in Fig.~\ref{fig:introduction}, we propose the novel Diverse GIB (DGIB) objective function, which incorporates a Determinant Point Process (DPP)~\cite{kulesza2012determinantal} constraint into GIB. \chenn{DPP quantifies diversity by measuring differences between core subgraphs through the determinant of the inner product of their subgraph representations.} 
Second, to encode \chenn{both high-order and low-order pair-wise structural} information from the candidate core subgraphs for prediction, \chenn{DGIB4SL employs a motif-based GNN encoder\cite{chen2023motif}}. Specifically, it uses 13 motif-based adjacency matrices to capture the high-order structure of a gene pair's neighborhood, \chenn{followed by} a GNN with injective concatenation to combine motif-wise representations \chenn{and produce} the final representation of the core graph. 
\chenn{We summarize our key contributions in the following.}
\begin{itemize}
    \item \chenn{We employ the GIB principle to  define a core subgraph, providing a principled alternative to attention weights  which often exhibit instability and bias toward frequent edges.} 
    \item \chenn{We extend the GIB objective to handle data with multiple core subgraphs, resulting in DGIB, which serves as the  objective for our DGIB4SL model.}
    \item \chenn{We use a motif-basd GNN encoder in DGIB4SL to capture both low- and high-order structures in node neighborhoods, ensuring reliable high-order structures in explanations.}
    \item \chenn{Experimental results demonstrate that our DGIB4SL  outperforms state-of-the-art methods in both accuracy and explanation diversity.}
\end{itemize}

\section{Related work}
SL prediction methods can be categorized into three types: statistical inference methods, network-based methods, and supervised machine learning methods~\cite{WangZHZZYDWWZLW22}.
Statistical methods\chenn{~\cite{sinha2017systematic,yang2021mapping,staheli2024predicting,liany2024aster}}, such as \chenn{ASTER~\cite{liany2024aster}}, rely on  predefined biological rules, which limit their applicability in complex systems due to strong underlying assumptions~\cite{DBLP:journals/tcbb/LiuWLLZ20, DBLP:journals/bioinformatics/LianyJR20, DBLP:journals/bmcbi/HuangWLOZ19}.
Network-based approaches\chenn{~\cite{megchelenbrink2015synthetic,ku2020integration,barrena2023synthetic}}, such as \chenn{iMR-gMCSs~\cite{barrena2023synthetic}},  improved reproducibility by analyzing pathway level interactions. However, their performance is often limited by noise and incomplete data.
With advancements in machine learning (ML), supervised techniques such as SVM~\cite{paladugu2008mining} and RF~\cite{li2019identification}, \chenn{and their combination~\cite{dou2024cssldb}} have been developed  to facilitate feature selection using manually crafted biological features. 
However, their dependency on manual feature engineering poses the risk of overlooking critical interactions. SL$^2$MF~\cite{DBLP:journals/tcbb/LiuWLLZ20} advances SL prediction by decomposing SL networks into matrices, offering a structured approach. However, its reliance on linear matrix decomposition struggles to capture the inherent complexity of SL networks. 
To overcome these limitations, deep learning methods\chenn{~\cite{DBLP:journals/bioinformatics/CaiCFWHW20,9947282,lai2021predicting,DBLP:journals/bioinformatics/LongWLZKLL21,DBLP:journals/titb/HaoWFWCL21,DBLP:journals/bioinformatics/LongWLFKCLL22,wang2022nsf4sl,zhang2024mpasl,zhang2024prompt,fan2023multi}} are developed. For example, DDGCN~\cite{DBLP:journals/bioinformatics/CaiCFWHW20}, the first GNN-based model, employs GCNs with dual-dropout to mitigate SL data sparsity. Similarly, 
\chenn{MPASL~\cite{zhang2024mpasl} improves gene representations by capturing SL interaction preferences and layer-wise differences on heterogeneous graphs.}
Although  many SL gene pairs have been identified, few of them  have been applied to cancer treatment. 
Understanding the underlying biological mechanisms is crucial for developing SL-based cancer therapies. Unfortunately, most ML models lack the capability to fully explain SL mechanisms. To address this, methods incorporating prior knowledge into the above models through knowledge graph (KG) have been proposed~\cite{wang2021kg4sl,liu2022pilsl,wang2022nsf4sl,zhang2023kr4sl,zhu2023slgnn}. Most of these methods utilize attention mechanisms to identify important edges~\cite{wang2021kg4sl,liu2022pilsl}, paths~\cite{zhang2023kr4sl}, or factors (subsets of relational features)~\cite{zhu2023slgnn} within KG to explain the  mechanisms underlying SL. For example, KR4SL~\cite{zhang2023kr4sl} encodes structural information, textual semantic information, and sequential semantics to generate gene representations and leverages attention to highlight important edges across hops to form paths as explanations. Similarly, SLGNN~\cite{zhu2023slgnn} focuses exclusively on KG data for factor-based gene representation learning, where relational features in the KG constitute factors, and attention weights are used to identify the most significant ones.
However, attention weights are often unstable, frequently assigning higher weights to frequent edges~\cite{DBLP:journals/tkde/LiSCZYX24}, and typically provide only a single explanation per sample. Additionally, these methods struggle to  capture high-order structures for prediction. To address these issues,  DGIB4SL replaces attention mechanisms with graph information bottlenecks to identify key edges and employs motif-based encoders along with DPP to encode high-order structures and generate multiple explanations. For further details on explainability in GNNs, please refer to our Appendix.

\section{Preliminaries}
\subsection{Notations and problem formulation}
An undirected SL graph is denoted by $G^{\text{SL}} = (\mathcal{V}^\text{SL}, \mathcal{Y}^\text{SL}, X^{\text{SL}})$, with the set of nodes (or genes) $\mathcal{V}^\text{SL}$, the set of edges or SL interactions $\mathcal{Y}^\text{SL}$ and the node feature matrix $X^{\text{SL}}$$\in$$\mathbb R^{|\mathcal{V}^{\text{SL}}|\times d_0}$. 
In addition to the SL interactions, we also have external knowledge about the functions of genes. 
We represent this information as a directed Knowledge Graph (KG) $G^{\text{KG}}$$=$$\{(h, r, t) | h, t \!\!\in\!\! \mathcal{V}^{\text{KG}}, r \!\!\in\!\! \mathcal{R}^{\text{KG}}\}$ and let $X^{\text{KG}}$ denote the node features associated with $G^{\text{KG}}$, where $\mathcal{V}^{\text{KG}}$ is a set of entities, and $\mathcal{R}^{\text{KG}}$ is a set of relations. To achieve the goals outlined later, we define $G$$=$$(A, X, E)$, where $A$ is the adjacency matrix, $X$$\in$$\mathbb R^{|\mathcal{V^{\text{SL}}} \cup \mathcal{V}^{\text{KG}}| \times d_0}$ is the node feature matrix  and $E$$\in$$\mathbb R^{|\mathcal{R}^{\text{KG}} \cup \{\text{``SL''}\}|\times d_1}$ is the edge feature matrix. Graph $G$ represents the directed \emph{joint} \emph{graph} of $G^{\text{SL}}$ and $G^{\text{KG}}$, constructed by mapping genes from $G^{\text{SL}}$ to entities in $G^{\text{KG}}$, and adding edges labeled ``SL''  for corresponding gene pairs based on $G^{\text{SL}}$.  We use $(T)_{ij}$ to represent the element at the $i$-th row and the $j$-th column of a matrix $T$, and $(T)_{i*}$ to represent the $i$-th row of the matrix. A comprehensive list of the mathematical notations used in this paper is provided in Table S1 in the Appendix.

In this paper, we investigate the problem of interpretable synthetic lethality prediction, which aims to extract a local subgraph around two target genes
for link prediction, potentially in an end-to-end fashion. Formally, given the joint graph $G$ that combines the SL graph and the KG, and a  pair of genes $u$ and $v$, we first collect their
$t$-hop neighborhoods from $G$  for each gene, $\mathcal{N}_t(u)$$=$$\{s|\text{dist}(s, u) \le t\}$,  $\mathcal{N}_t(v)$$=$$\{s|\text{dist}(s, v) \le t\}$, where $\text{dist}(\cdot, \cdot)$ denotes the shortest distance between two nodes. We then take the intersection of nodes between their neighborhoods to construct a pairwise \textbf{enclosing graph}~\cite{liu2022pilsl} $G^{uv}$$=$$\{(i, r, j)| i, j$$\in$$\mathcal{N}_t(u) \cap \mathcal{N}_t(v), r\in \mathcal{R}^{\text{KG}}\cup \{\text{``SL''}\}\}$. Our goal is to learn a function $f:$$\{G^{uv}|u,v \!\in\! \mathcal{V}^{\text{SL}}\} \!\to\! \{\widetilde{G}^{uv}|\widetilde{G}^{uv} \subseteq G^{uv}\}$, which maps the enclosing graphs $G^{uv}$ to optimized subgraph $\widetilde G^{uv}$, and learn a binary classifier $g_\theta(\widetilde G^{uv})$ parameterized by $\theta$ for SL prediction based on the optimized subgraph $\widetilde G^{uv}$.

\subsection{Information bottleneck}
In machine learning, it is crucial to determine which parts of the input data should be preserved and which should be discarded. Information bottleneck (IB)~\cite{DBLP:journals/corr/physics-0004057} offers a principled approach for addressing this challenge by compressing the source random variable to keep the information relevant for predicting the target random variable and discarding target-irrelevant information.
\begin{definition} (Information Bottleneck (IB)).
\small
   Given random variables $Q$ and $Y$, the Information Bottleneck principle aims to compress $Q$ to a bottleneck random variable $B$, while keeping the information relevant for predicting $Y$:
   \begin{equation}
       \min_B  \underbrace{-I(Y;B)}_{\text{Prediction}} + \underbrace{\beta I(Q;B)}_{\text{Compression}},
   \end{equation}
where $\beta$$>$$0$ is a Lagrangian multiplier to balance the two
mutual information terms.
\end{definition}

Recently, the IB principle has been applied to learn a bottleneck graph named IB-Graph  for the input graph~\cite{DBLP:conf/nips/WuRLL20}, which keeps \emph{minimal sufficient information} in terms of the graph's data. In our context of SL prediction, the IB-Graph is defined as follows.
\begin{definition} (IB-Graph). 
    For an enclosing graph $G^{uv}$$=$$(A^{uv},$ $X^{uv},$$E^{uv})$ around a pair of genes $u$ and $v$ and the associated label information $Y$, the optimal subgraph $\widetilde G^{uv}$$=$$(\widetilde A^{uv}, \widetilde X^{uv})$ found by Information Bottleneck principle is called  IB-Graph if
    \begin{equation}\label{equ:gib}
    \small
        \widetilde G^{uv} = \underbrace{\arg\min_{\widetilde G^{uv}} -I(Y;\widetilde G^{uv}) + \beta I(G^{uv};\widetilde G^{uv})}_{\text{Graph Information Bottleneck, GIB}},
    \end{equation}
    where $\widetilde A^{uv}$ and $\widetilde X^{uv}$ are the task-relevant adjacency matrix  and the node feature matrix of $\widetilde G^{uv}$, respectively.
\end{definition}
Intuitively, GIB (Eq.~\ref{equ:gib}) aims to learn the core subgraph of the input graph $G^{uv}$, discarding information from $G^{uv}$ by minimizing the mutual information $I(G^{uv}; \widetilde G^{uv})$, while preserving target-relevant information by maximizing the mutual information $I(Y;\widetilde G^{uv})$.

\section{Methods}
\subsection{Overview}
\begin{figure*}[t!]
	\centering
	\includegraphics[width=2\columnwidth]{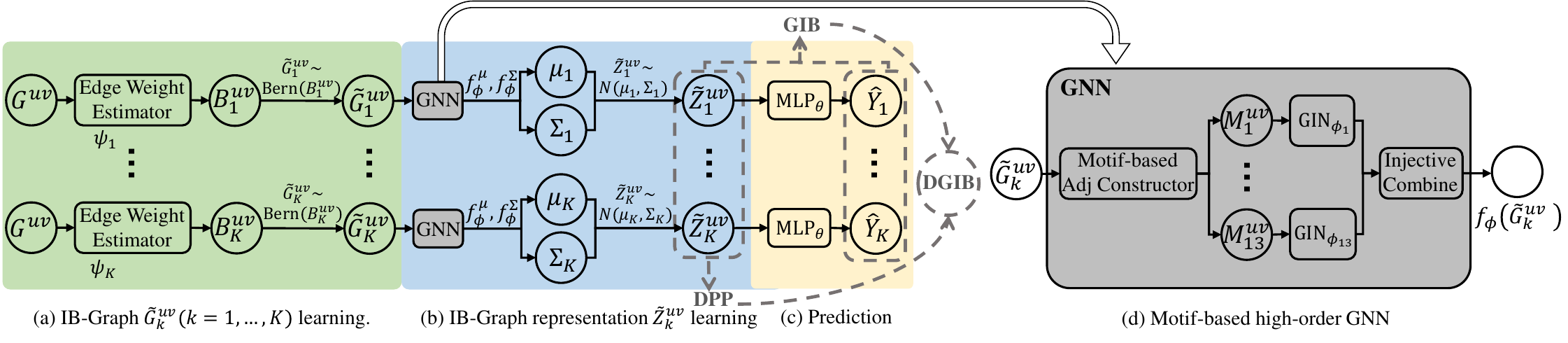} 
	\caption{Overview of DGIB4SL. DGIB4SL takes the enclosing graph data $G^{uv} = (A^{uv}, X^{uv}, E^{uv})$ around genes $u$ and $v$ as inputs,  throughout the phases (a),(b), and (c), and outputs the interaction confidence of the gene pair $(u, v)$ and $K$ IB-graphs  $\widetilde G^{uv}_1$, ..., $\widetilde G^{uv}_K$ that captures the high-order graph structure. 
    \chenn{In phase (a), an IB-graph $\widetilde G^{uv}_k$ is generated by injecting random noise to select important edges, with edge weights \( B^{uv}_k \) estimated from \( G^{uv} \) using  the edge weight estimation module (Eq. S11). $B^{uv}_k$ serves as the parameter for a multi-dimensional Bernoulli distribution, from which an adjacency matrix of $\widetilde G^{uv}_k$ is sampled. 
    In phase (b), IB-graph representations are learned via variational estimation. Each IB graph data $\widetilde G^{uv}_k$ is passed through the same motif-based GNN \( f_\phi \) (Eq. 10) to obtain a distribution from which a representation \( \widetilde Z^{uv}_k \) is sampled. The motif-based GNN, shown in subfigure (d), projects the IB-graph into 13 motif-based matrices. Each motif-based matrix $M^{uv}_k$ is processed by a different GIN encoder to produce motif-wise representations, which are then concatenated (Eq. 15). 
    In phase (c), each IB-graph representation is passed through an MLP-based classifier to make \( K \) predictions (Eq. 11). 
    During training, the representations and predictions are used to compute DPP and GIB, which are jointly optimized in DGIB4SL.}
    } 
\label{fig:model}
\end{figure*}

In this section, we present Diverse Graph Information Bottleneck for Synthetic Lethality (DGIB4SL), an interpretable SL prediction framework that incorporates a diversity constraint using Determinantal Point Process (DPP) and GIB objective to generate multiple explanations called IB-graphs for the same gene pair.
\chen{The framework consists of three key components: IB objective formulation, motif-based DGIB estimation, and prediction.} 
\chen{First}, we introduce our DGIB objective and derive its tractable upper bound. 
\chen{Next}, given that most of the existing IB estimation approaches fail to capture high-order structure information, we propose a novel motif-based DGIB estimation method, which involves three phases: IB-graph learning through random noise injection  to select significant edges, graph representation learning (GRL), and prediction, as shown in Fig.~\ref{fig:model}(a)-(c). 
In the GRL phase, 
we employ the motif-wise representation learning \chen{method}~\cite{chen2023motif} to implement the GNN module in Fig.~\ref{fig:model}(b), enabling the capture of high-order structures in IB-graph, as illustrated  in Fig.~\ref{fig:model}(d).

\subsection{Diverse graph information bottleneck}\label{sec:dgib_frame}
We now present our main result, which demonstrates how to generate $K$$\ge$$1$ different IB-graphs for any gene pair $(u, v)$, denoted as $\{\widetilde G^{uv}_1, ..., \widetilde G^{uv}_K\}$. We first reduce this problem to a special case of the \emph{subset selection problem} where diversity is preferred, i.e., the problem of balancing two aspects: (i) each selected subgraph should satisfy the definition of an IB-Graph; and (ii) the selected subgraphs should be diverse as a group so that the subset is as informative as possible. 

Determinantal Point Process (DPP)~\cite{kulesza2012determinantal} is an elegant and effective probabilistic model designed to address one key aspect of the above problem: diversity. Formally, let ${\mathcal{G}}^{uv}$ denote the set of all possible subgraphs of a graph $G^{uv}$. A point process $P$ defined on the ground set ${\mathcal{G}}^{uv}$ is a probability measure over the power set of ${\mathcal{G}}^{uv}$. 
$P$ is called a DPP if, for any subset $\{\widetilde G^{uv}_1, ..., \widetilde G^{uv}_K\} \subseteq \mathcal{G}^{uv}$, the probability of selecting this subset is given by
\begin{equation}\label{equ:ddp}
\small
    P(\{\widetilde G^{uv}_1, ..., \widetilde G^{uv}_K\}) \propto \text{det}(L^{uv}),
\end{equation}
where $\text{det}(\cdot)$  represents the determinant of a given matrix and  $L^{uv}$$\in$$\mathbb R^{K \times K}$ is a real, positive semidefinite (PSD) matrix and thus \chen{there} exists a matrix $U\in\mathbb R^{K \times d_3}$ such that
\begin{equation}\label{equ:equ_L}
\small
    L^{uv} = UU^{\mathrm{T}}, (U)_{k} = \text{GRL}(\widetilde G^{uv}_k),
\end{equation}
where $(U)_k$$\in$$\mathbb R^{d_3}$ is the graph representation of the $k$-th IB-graph $\widetilde G^{uv}_k$ and $\text{GRL}$ denotes the graph representation learning module. More details about the  $\text{GRL}$ module and intuitions about the ability of Eq.~\ref{equ:ddp} to measure diversity are  provided in Eqs.~\ref{equ:qz}-\ref{equ:u_z} and Appendix, respectively. 

To learn $K$ different subgraphs from the enclosing graph $G^{uv}$ for the gene pair $(u,v)$ that balance diversity with the IB-graph definition, we  introduce the Diverse Graph Information Bottleneck (DGIB) objective function, formulated as follows:
\begin{equation}\label{equ:gib_ddp}
\small
    \min_{\widetilde G^{uv}_1, ..., \widetilde G^{uv}_K} \frac{1}{K}\sum^K_{k=1} \underbrace{-I(Y;\widetilde G^{uv}_k) + \beta_1 I(G^{uv};\widetilde G^{uv}_k)}_{\text{Graph Information Bottleneck}} - \beta_2 \underbrace{\text{det}(L^{uv})}_{\text{DPP}},
\end{equation}
where $\beta_2$$>$$0$ is a Lagrangian multiplier to trade off GIB and DPP terms. Intuitively, the GIB term focuses on learning multiple IB-graphs from the input graph $G^{uv}$, while the DPP term ensures that these IB-graphs are as different as possible.

Due to the non-Euclidean nature of graph data and the intractability of mutual information, it is challenging to optimize the DGIB objective in Eq.~\ref{equ:gib_ddp} directly. Therefore, we adopt the approach of Sun et al.~\cite{DBLP:conf/aaai/Sun0P0FJY22} to derive  tractable variational upper bounds for $-I(Y;\widetilde G^{uv}_k)$ and $I(G^{uv};\widetilde G^{uv}_k)$: 
\begin{equation}\label{equ:gib_ddp2}
\small
\begin{aligned}
    \min_{\widetilde G^{uv}_1, ..., \widetilde G^{uv}_K}  
    &\frac{1}{K}\sum^K_{k=1} \underbrace{\chenn{-\mathbb E_{Y, \widetilde G^{uv}_k} \Big[\log q_\theta(Y|\widetilde G^{uv}_k)\Big]}}_{\text{Upper bound of $-I(Y;\widetilde G^{uv}_k)$}}\\
    &+ \beta_1\frac{1}{K}\sum^K_{k=1} \underbrace{\mathbb E_{G^{uv}} [D_{\text{KL}}(q_\phi(\widetilde G^{uv}_k|G^{uv})||q(\widetilde G^{uv}_k)]}_{\text{Upper bound of $I(G^{uv};\widetilde G^{uv}_k)$}}\\
    &- \beta_2 \text{det}(L^{uv}).
\end{aligned}
\end{equation}
Detailed proof of Eq.~\ref{equ:gib_ddp2}  is given in Appendix.

~\\\textbf{Remark 1:} 
Each explanation or IB-graph generated by DGIB4SL consists of a single core subgraph rather than a combination of multiple core subgraphs. \chennn{In datasets with tens of thousands of gene pairs, as used in our experiments, individual core subgraphs are more likely to be shared across different enclosing graphs than combinations of multiple core subgraphs. This is because the probability of a specific combination being repeatedly shared decreases exponentially with its complexity. In contrast, the structural simplicity of individual core subgraphs makes them more likely to be shared.} Minimizing the compression term in DGIB allows DGIB4SL to select individual core subgraphs with higher \chennn{shared frequency}.

\subsection{High-order motif-based DGIB estimation}\label{sec:dgib_motif}
We now address another key question of this work: how to compute the DGIB upper bound in Eq.~\ref{equ:gib_ddp2} without losing the high-order information, \chen{which is crucial for generating trustworthy explanations. For instance, in a KG, a gene’s functional relevance often  depends on high-order structures, such as cooperative pathways or shared regulatory targets among its neighbors.
Ignoring these structures can result in misleading explanations. 
To overcome this,} we propose a novel high-order motif-based DGIB  estimation method DGIB4SL.

\subsubsection{Mutual information estimation}
We first outline the general procedure for estimating the  DGIB upper bound defined in Eq.~\ref{equ:gib_ddp2}  which is largely analogous to previous work~\cite{tian2020makes,DBLP:conf/aaai/Sun0P0FJY22}. This procedure involves learning the $k$-th IB-graph $\widetilde G^{uv}_k$  from the enclosing graph $G^{uv}$ and driving its representation $\widetilde Z^{uv}_k$$\in$$\mathbb R^{d_3}$ through a graph representation learning function ($\text{GRL}$), such that $\widetilde Z^{uv}_k$$=$$\text{GRL}(\widetilde G^{uv}_k)$, assuming no information is lost during this transformation. Under this assumption, $I(Y;\widetilde Z_k^{uv})$$\approx$$I(Y;\widetilde G_k^{uv})$, $I(G^{uv};\widetilde Z_k^{uv})$$\approx$$I(G^{uv};\widetilde G_i^{uv})$. Consequently,  the  DGIB upper bound, which DGIB4SL aims to minimize, is expressed as:
\begin{equation}\label{equ:gib_ddp3}
\small
\begin{aligned}
    \min_{\widetilde G^{uv}_1, ..., \widetilde G^{uv}_K}  
    -&\frac{1}{K}\sum^K_{k=1} \mathbb E_{Y, \widetilde G^{uv}_k} \Big[\log q_\theta(Y| \widetilde Z^{uv}_k)\Big] - \beta_2 \text{det}(L^{uv})\\
    &+ \beta_1\frac{1}{K}\sum^K_{k=1} \mathbb E_{G^{uv}} [D_{\text{KL}}(q_\phi( \widetilde Z^{uv}_k|G^{uv})||q(\widetilde Z^{uv}_k)]\\
\end{aligned}
\end{equation}
To calculate Eq.~\ref{equ:gib_ddp3}, we follow a two-step process. In \textbf{Step 1}, we estimate a \chenn{IB-graph} 
$\widetilde G_k^{uv}$ \chennn{based on all the subgraphs} from $G^{uv}$. In \textbf{Step 2}, we implement the $\text{GRL}$ function to infer the graph representation $\widetilde Z_k^{uv}$ of $\widetilde G_k^{uv}$ and feed $\widetilde Z_k^{uv}$ into Eq.~\ref{equ:gib_ddp3}.\\
(\textbf{Step 1: IB-graph $\widetilde G_k^{uv}$ Learning}) We compress the information of $G^{uv}$$=$$(A^{uv}, X^{uv}, E^{uv})$ via noise injection to estimate  the $k$-th IB-graph $\widetilde G^{uv}_k$$=$$(\widetilde A^{uv}_k, \widetilde X^{uv}_k)$. To construct $\widetilde A^{uv}_k$, we model  all potential edges of the subgraph as mutually independent Bernoulli random variables. The parameters of these variables are determined by the learned important weights $B_k^{uv}$$\in$$\mathbb R^{|\mathcal{V}^{uv}|\times|\mathcal{V}^{uv}|}$, where $\mathcal{V}^{uv}$ denotes the set of entities in $G^{uv}$:
\begin{equation}\label{equ:AB}
\small
    \widetilde A^{uv}_k = \bigcup_{i,j \in \mathcal{V}^{uv}}\left\{(\widetilde A^{uv}_k)_{i, j} \sim \operatorname{Bernoulli}\left((B_k^{uv})_{i,j}\right)\right\},
\end{equation}
where $(B_k^{uv})_{i,j}$ represents the importance weight or sampling probability for the entity pair $(i,j)$.
The computation of  $B_k^{uv}$ (corresponding  to $\psi_k$ in Figure 1) is jointly  optimized   with relational graph learning, following the approach of  Wang et al.~\cite{wang2019kgat}. Further details are provided in Appendix. To sample the IB-graph, we employ the concrete relaxation~\cite{DBLP:conf/iclr/JangGP17} for the  Bernoulli distribution. Additionally, we construct $\widetilde X_k^{uv}$ to be the same as $X^{uv}$ since no nodes are removed during the construction of $\widetilde A_k^{uv}$. \chenn{An example of the IB-graph construction is provided in the Appendix.}
\\(\textbf{Step 2: IB-Graph Representation Learning and Prediction})  
Using the previously constructed $\widetilde G_k^{uv}$$=$$(\widetilde A_k^{uv}, \widetilde X_k^{uv})$, we compute the prediction, diversity and KL terms in Eq.~\ref{equ:gib_ddp3} by implementing $\widetilde Z_k^{uv}$$=$$\text{GRL}(\widetilde G_k^{uv})$  through variational inference. For the KL term $D_{\text{KL}}(q_\phi( \widetilde Z^{uv}_{k}|G^{uv})||q(\widetilde Z^{uv}_k))$, we treat  the prior $p(\widetilde Z^{uv}_k)$ and the posterior $q_\phi(\widetilde Z^{uv}_k|G^{uv})$ as  parametric Gaussians and thus this  term has an analytic solution:
\begin{equation}\label{equ:qz}
\small
\begin{aligned}
    &p(\widetilde Z_k^{uv}) := N(\mu_0, \Sigma_0),\\
    &q_\phi(\widetilde Z_k^{uv}|G^{uv}) := N(f^\mu_{\phi}(\widetilde A_k^{uv}, \widetilde X_k^{uv}), f^\Sigma_\phi(\widetilde A_k^{uv}, \widetilde X_k^{uv})),
\end{aligned}
\end{equation}
where the outputs  $f^\mu_\phi(\cdot)$$\in$$\mathbb R^{d_3}$ and $f^\Sigma_\phi(\cdot)$$\in$$\mathbb R^{d_3 \times d_3}$ represent  the mean vector and the diagonal covariance matrix of the distribution for the graph embedding $\widetilde Z_k^{uv}$$\in$$\mathbb R^{d_3}$ of $\widetilde G_k^{uv}$, respectively. We model $f_\phi(\cdot)$ as a GNN parameterized by the weights $\phi$ with $2d_3$-dimensional output and a readout or pooing operator. 
The first $d_3$ dimensions of this GNN’s output correspond to $f^\mu_\phi(\cdot)$, while the remaining $d_3$ dimensions correspond to $f^\Sigma_{\phi}(\cdot)$, formally expressed as:
\begin{equation}\label{equ:gnn_graph}
\small
   (f^\mu_{\phi}(\widetilde A_k^{uv}, \widetilde X_k^{uv}), f^\Sigma_{\phi}(\widetilde A_k^{uv}, \widetilde X_k^{uv})) = \mathrm{readout}(\mathrm{GNN}(\widetilde A_k^{uv}, \widetilde X_k^{uv};\phi)).
\end{equation}
We treat $p(\widetilde Z_k^{uv})$$:=$$N(\mathbf{0}, I)$ as a fixed $d_3$-dimensional spherical Gaussian distribution. 
To compute the prediction term $\mathbb E_{Y, \widetilde G^{uv}_k} [\log q_\theta(Y| \widetilde Z^{uv}_k)]$, we adopt the equivalent cross entropy loss function $L^{\mathrm{CE}}$. The conditional distribution  $q_\theta(Y| \widetilde Z^{uv}_k)$ is implemented using a 2-layer perceptron in this work, parameterized by trainable weights  $\theta$, as described below.
\begin{equation}
\small
    -\mathbb E_{Y, \widetilde G_k^{uv}} \Big[\log q_\theta(Y| \widetilde Z_k^{uv})\Big] = L^{\mathrm{CE}}(Y, \mathrm{MLP}(\widetilde Z_k^{uv};\theta)).
\end{equation}
Finally, for the diversity term $\text{det}(L^{uv})= \text{det}(UU^\mathrm{T})$ (Eq.~\ref{equ:gib_ddp3}, Eq.~\ref{equ:equ_L}), the matrix $U$ is constructed by arranging the $K$ IB-graph representations as its rows. Specifically, 
\begin{equation}\label{equ:u_z}
\small
    U = (\widetilde Z_1^{uv}, ..., \widetilde Z_K^{uv})^{\mathrm{T}}.
\end{equation}

\subsubsection{Generating high-order graph representation via motif}\label{sec:motif}
Most methods struggle to satisfy the no-information-loss assumption of the above mutual information estimation framework, since their GNN implementation in Eq.\ref{equ:gnn_graph} often fails to capture the high-order structure of the estimated explanation. Inspired by MGNN~\cite{chen2023motif}, we reduce this to a problem of enhancing the model's representation power beyond the $1$-dimensional Weisfeiler-Leman ($1$-WL) graph isomorphism test~\cite{DBLP:conf/iclr/XuHLJ19}.  
Specifically, the 1-WL test distinguishes graph structures by iteratively compressing node neighborhood information into unique labels, making it a widely recognized tool for evaluating the expressive power of GNNs~\cite{DBLP:conf/iclr/XuHLJ19,maron2019provably,chen2023motif}.

We first formalize three key definitions underlying our approach, starting with the notion of a network motif.
\begin{figure}[t!]
	\centering
	\includegraphics[width=0.8\columnwidth]{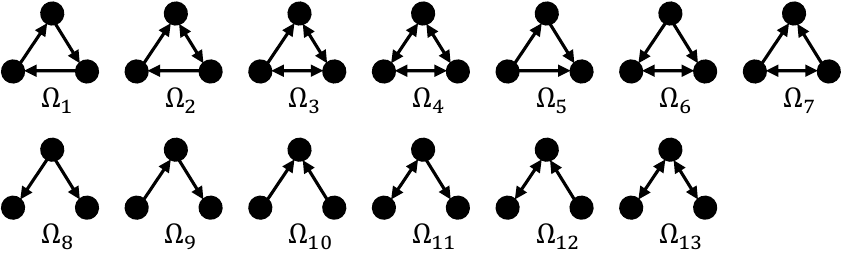} 
	\caption{All 3-node motifs in a directed and unweighted graph.}
	\label{fig:motif}
\end{figure}
\begin{definition}\label{def:motif} (Network motif). 
    A motif is a connected graph of $d_5$ nodes ($d_5 > 1$), with a $d_5 \times d_5$ adjacency matrix $C$ containing binary elements $\{0, 1\}$.
\end{definition}
Let $\Omega_i$ denote different motifs and $C_i$ represent the corresponding associated matrices. An example of all possible 3-node motifs is shown in Fig.\ref{fig:motif}. Chen et al.~\cite{chen2023motif} demonstrated that $3$-node motifs are sufficiently expressive to capture graph structures.   Thus, we only use motifs with $3$ nodes in this work.

\begin{definition} (Motif set). 
    The motif set of a $3$-node motif $\Omega_i$ in a directed graph $G=(A,X)$ is defined by 
    \begin{equation}
    \small
        \mathcal{M}_i = \{ V | V \in \mathcal{V}^3,   A^V = C_i\}, 
    \end{equation}
    where $V$ is a tuple containing 3 node indices and $A^V$ is the $d_5 \times d_5$ adjacency matrix of the subgraph induced by $V$. 
\end{definition}
For example, the motif set of $\Omega_5$ in Fig.~\ref{fig:introduction} can be $\{(\text{HR}, \text{Trapped replication fork}, \text{DNA damage}),$  $(\text{HR}, \text{DSB}, \text{DNA}$ $\text{damage})\}$. Based on the motif set, we define the operator $\text{set}(\cdot)$ to transform an ordered tuple into an unordered set, e.g., $\mathrm{set}((v_1, ..., v_3))$$=$$\{v_1, ..., v_3\}$. Using this operator, the motif-based adjacency matrix is defined as follows.
\begin{definition} (Motif-based adjacency matrix). For a graph $G=(A,X)$, 
     a motif-based adjacency matrix $M_i$ of $G$ in terms of a given motif $\Omega_i$ is defined by
     \begin{equation}
     \small
         (M_i)_{j,l} = \sum_{V \in \mathcal{M}_i} \mathbb I( \{j,l\}\subset \mathrm{set}(V)).
     \end{equation}
\end{definition}
Intuitively, $(M_i)_{j,l}$ denotes the number of times  nodes $j$ and $l$ are connected through an element of the given motif set $\mathcal{M}_i$. The roles of these definitions will be discussed in Eq.~\ref{equ:motif_rst}.

To generate graph embeddings with greater expressive power than the $1$-WL test, Chen et al.\cite{chen2023motif} demonstrated that associating node or graph embeddings with different motif structures and combining these embeddings using an injective function effectively captures high-order \chenn{and low-order} graph structure.
Specifically, we use a 2-layer GIN\cite{DBLP:conf/iclr/XuHLJ19} as the underlying GNN  and employ different GINs to encode the structure of 13 motifs in $\widetilde G^{uv}$, producing node embeddings through motif-based adjacency matrices. Then, the motif-wise embeddings are  combined via injective concatenation.   Mathematically, we construct the GNN module in  Eq.\ref{equ:gnn_graph} as 
\begin{equation}\label{equ:motif_rst}
\small
    \mathrm{GNN}(\widetilde A_k^{uv}, \widetilde X_k^{uv};\phi) := \|_{i=1}^{13} \mathrm{GIN}( M^{uv}_{i}, \widetilde X^{uv}_{k};\phi_i),
\end{equation}
where $M^{uv}_i$ are the motif-based adjacency matrix of $\widetilde G^{uv} = (\widetilde A_k^{uv}, \widetilde X_k^{uv})$ in terms of a given motif $\Omega_i$ and  $\|$ denotes a concatenation function. 
In summary, Eq.~\ref{equ:motif_rst} preserves high-order \chenn{and low-order pair-wise} structural information when calculating DGIB, enhancing the reliability of the high-order structure in an IB-graph.

\section{Results}
\subsection{Experimental setup}
\subsubsection{Datasets and baselines}
To evaluate the effectiveness of our  DGIB4SL, we utilized the dataset provided by the Synthetic Lethality Benchmark (SLB)~\cite{feng2023benchmarking}. The dataset is collected from  SynLethDB 2.0, a comprehensive repository of SL data, and includes 11 types of entities and 27 relationships. It contains  35,913 human SL gene pairs involving 9,845 genes, along with a KG named SynLethKG, which comprises 54,012 nodes and 2,233,172 edges. Additional details on SynLethKG can be found in Tables S2--S3 in Appendix.

We evaluated two categories of methods, selecting thirteen recently published methods. These include three matrix factorization (MF)-based methods: GRSMF~\cite{DBLP:journals/bmcbi/HuangWLOZ19}, SL2MF~\cite{DBLP:journals/tcbb/LiuWLLZ20}, and CMFW~\cite{DBLP:journals/bioinformatics/LianyJR20}, and ten GNN-based methods: DDGCN~\cite{DBLP:journals/bioinformatics/CaiCFWHW20}, GCATSL~\cite{DBLP:journals/bioinformatics/LongWLZKLL21}, SLMGAE~\cite{DBLP:journals/titb/HaoWFWCL21}, MGE4SL~\cite{lai2021predicting}, PTGNN~\cite{DBLP:journals/bioinformatics/LongWLFKCLL22}, KG4SL~\cite{wang2021kg4sl}, PiLSL~\cite{liu2022pilsl}, NSF4SL~\cite{wang2022nsf4sl}, KR4SL~\cite{zhang2023kr4sl}, and SLGNN~\cite{zhu2023slgnn}. 
Among these, KG4SL, PiLSL, NSF4SL, KR4SL, and SLGNN integrate knowledge graphs (KGs) into the generation of node representations. Detailed descriptions of these baselines can be found in Appendix.

\subsubsection{Implementation details}
We evaluated our method using 5-fold cross-validation by splitting the gene pairs and using 4 ranking metrics: Normalized Discounted Cumulative Gain (NDCG@\chenn{C}), Recall@\chenn{C}, Precision-@\chenn{C}, and Mean Average Precision (MAP@\chenn{C}). NDCG@\chenn{C} measures the positioning of known SL gene pairs within the model's predicted list, while Recall@\chenn{C} and Precision@\chenn{C} assess the model’s ability to identify relevant content and rank the top \chenn{C} results accurately, respectively. MAP@\chenn{C} provides a comprehensive evaluation by combining precision and ranking across multiple queries, averaging the precision at each relevant prediction up to the \chenn{C}-th position. In this study, we evaluated these metrics using the top \chenn{C}=10 and top \chenn{C}=50 predictions. The coefficients $\beta_1$ and $\beta_2$ in Eq.~\ref{equ:gib_ddp2} were set $\beta_1$$=$$\beta_2$$=$$10^{-4}$.  More details on data preprocessing,  hyperparameters settings for DGIB4SL and baseline implementations are provided  in  Appendix.

\begin{table*}[t!]
  \centering
  \caption{Performance of various methods in terms of NDCG and Recall under 5-fold cross-validation. The best results are highlighted in bold, while the second-best results are underlined. \chenn{Values in parentheses indicate paired $t$-test $p$-values comparing baselines with DGIB4SL.}}
    \begin{tabular}{lllll}
    \toprule
          & NDCG@10  & NDCG@50 & Recall@10 & Recall@50\\
    \midrule
    GRSMF~\cite{DBLP:journals/bmcbi/HuangWLOZ19} &   0.2844 \chenn{($2.06 \times 10^{-7}$)}     &   0.3153 \chenn{($9.20 \times 10^{-9}$)}    &  0.3659 \chenn{($8.04 \times 10^{-7}$)}  & 0.4460 \chenn{($3.07 \times 10^{-4}$)}\\
    SL$^2$MF~\cite{DBLP:journals/tcbb/LiuWLLZ20}   &   0.2807 \chenn{($7.79 \times 10^{-8}$)}     &   0.3110 \chenn{($7.49 \times 10^{-8}$)}    &  0.2642 \chenn{($8.44 \times 10^{-7}$)}  & 0.3401 \chenn{($2.95 \times 10^{-7}$)}\\
    CMFW~\cite{DBLP:journals/bioinformatics/LianyJR20}   &   0.2390 \chenn{($1.24 \times 10^{-7}$)}    &   0.2744 \chenn{($7.16 \times 10^{-7}$)}    &  0.3257 \chenn{($2.15 \times 10^{-6}$)} & 0.4097 \chenn{($4.14 \times 10^{-5}$)} \\
    \midrule
    DDGCN~\cite{DBLP:journals/bioinformatics/CaiCFWHW20} &  0.1568 \chenn{($1.94 \times 10^{-7}$)}     &   0.1996 \chenn{($6.74 \times 10^{-8}$)}    &  0.2379 \chenn{($7.61 \times 10^{-9}$)}  & 0.3447 \chenn{($2.76 \times 10^{-5}$)} \\
    GCATSL~\cite{DBLP:journals/bioinformatics/LongWLZKLL21} &  0.2642 \chenn{($3.99 \times 10^{-7}$)}    &   0.2976 \chenn{($1.24 \times 10^{-6}$)}    &  0.3363 \chenn{($5.91 \times 10^{-6}$)} & 0.4203 \chenn{($3.67 \times 10^{-4}$)} \\
    SLMGAE~\cite{DBLP:journals/titb/HaoWFWCL21} &  0.2699 \chenn{($6.63 \times 10^{-7}$)}     &   0.3160 \chenn{($1.19 \times 10^{-6}$)}    &  0.3198 \chenn{($1.26 \times 10^{-5}$)}  & 0.4421 \chenn{($1.26 \times 10^{-3}$)}\\
    MGE4SL~\cite{lai2021predicting} &  0.0028 \chenn{($2.66 \times 10^{-9}$)}     &   0.0071 \chenn{($4.89 \times 10^{-9}$)}     &  0.0020 \chenn{($1.35 \times 10^{-8}$)} & 0.0085 \chenn{($6.13 \times 10^{-9}$)} \\
    PTGNN~\cite{DBLP:journals/bioinformatics/LongWLFKCLL22} &  0.2358 \chenn{($1.07 \times 10^{-7}$)}    &   0.2740 \chenn{($6.66 \times 10^{-7}$)}    &  0.3361 \chenn{($9.43 \times 10^{-6}$)}  & 0.4323 \chenn{($1.96 \times 10^{-4}$)}\\
    \midrule
    KG4SL~\cite{wang2021kg4sl} &  0.2505 \chenn{($1.04 \times 10^{-6}$)}    &   0.2853 \chenn{($1.05 \times 10^{-7}$)}     &  0.3347 \chenn{($4.40 \times 10^{-7}$)}  & 0.4253 \chenn{($2.48 \times 10^{-5}$)}\\
    PiLSL~\cite{liu2022pilsl} &  \underline{0.5166} \chenn{($3.32 \times 10^{-5}$)}    &   0.5175 \chenn{($3.28 \times 10^{-5}$)}    &  0.3970 \chenn{($3.99 \times 10^{-6}$)} & 0.4021 \chenn{($3.82 \times 10^{-6}$)}\\
    NSF4SL~\cite{wang2022nsf4sl} &  0.2279 \chenn{($2.65 \times 10^{-6}$)}    &   0.2706 \chenn{($4.50 \times 10^{-7}$)}     &  0.3526 \chenn{($2.96 \times 10^{-5}$)} & \underline{0.4624} \chenn{($1.14 \times 10^{-3}$)}\\
    KR4SL~\cite{zhang2023kr4sl} &  0.5105 \chenn{($2.19 \times 10^{-5}$)}    &   \underline{0.5248} \chenn{($5.79 \times 10^{-5}$)}    &  \underline{0.4131} \chenn{($6.97 \times 10^{-6}$)} & 0.4135 \chenn{($5.62 \times 10^{-6}$)}\\
    SLGNN~\cite{zhu2023slgnn} &  0.1468 \chenn{($4.49 \times 10^{-8}$)}     &   0.2004 \chenn{($1.85 \times 10^{-7}$)}    &  0.2154 \chenn{($1.90 \times 10^{-7}$)} & 0.3717 \chenn{($1.26 \times 10^{-4}$)}\\
    \midrule
    DGIB4SL  &  \textbf{0.5760}     &   \textbf{0.5766}    &  \textbf{0.5233}  & \textbf{0.5280}\\
    \bottomrule
    \end{tabular}%
  \label{tab:performance}%
\end{table*}%
\begin{table*}[t!]
  \centering
  \caption{Performance of various methods in terms of Precision and MAP under 5-fold cross-validation. The best results are highlighted in bold, while the second-best results are underlined. \chenn{Values in parentheses indicate paired $t$-test $p$-values comparing baselines with DGIB4SL.}}
    \begin{tabular}{lllll}
    \toprule
          & Precision@10 & Precision@50 & MAP@10 & MAP@50\\
    \midrule
    GRSMF~\cite{DBLP:journals/bmcbi/HuangWLOZ19} &   0.3683 \chenn{($1.02 \times 10^{-5}$)}     &   0.4461 \chenn{($1.42 \times 10^{-4}$)}     &  0.2568 \chenn{($7.15 \times 10^{-7}$)}  & 0.2521 \chenn{($9.92 \times 10^{-8}$)} \\
    SL$^2$MF~\cite{DBLP:journals/tcbb/LiuWLLZ20}   &   0.2694 \chenn{($3.53 \times 10^{-7}$)}     &   0.3407 \chenn{($3.45 \times 10^{-6}$)}     &  0.2861 \chenn{($1.35 \times 10^{-7}$)}  & 0.2769 \chenn{($1.68 \times 10^{-7}$)} \\
    CMFW~\cite{DBLP:journals/bioinformatics/LianyJR20}&   0.3267 \chenn{($8.57 \times 10^{-7}$)}     &   0.4098 \chenn{($4.37 \times 10^{-5}$)}     &  0.2043 \chenn{($5.65 \times 10^{-7}$)}  & 0.2069 \chenn{($1.13 \times 10^{-6}$)} \\
    \midrule
    DDGCN~\cite{DBLP:journals/bioinformatics/CaiCFWHW20}&   0.2385 \chenn{($1.60 \times 10^{-6}$)}     &   0.3447 \chenn{($1.41 \times 10^{-5}$)}    &  0.1280 \chenn{($6.26 \times 10^{-9}$)}  & 0.1321 \chenn{($3.74 \times 10^{-8}$)} \\
    GCATSL~\cite{DBLP:journals/bioinformatics/LongWLZKLL21}&   0.3372 \chenn{($5.36 \times 10^{-6}$)}     &   0.4204 \chenn{($1.03 \times 10^{-3}$)}      &  0.2354 \chenn{($8.96 \times 10^{-7}$)}   & 0.2382 \chenn{($7.66 \times 10^{-8}$)}  \\
    SLMGAE~\cite{DBLP:journals/titb/HaoWFWCL21}&   0.3222 \chenn{($3.58 \times 10^{-6}$)}      &   0.4422 \chenn{($7.76 \times 10^{-4}$)}      &  0.2514 \chenn{($2.10 \times 10^{-7}$)}   & 0.2469 \chenn{($9.86 \times 10^{-7}$)}  \\
    MGE4SL~\cite{lai2021predicting}&   0.0022 \chenn{($1.31 \times 10^{-8}$)}      &   0.0085 \chenn{($8.87 \times 10^{-9}$)}      &  0.0018 \chenn{($3.17 \times 10^{-9}$)}   & 0.0024 \chenn{($1.58 \times 10^{-8}$)}  \\
    PTGNN~\cite{DBLP:journals/bioinformatics/LongWLFKCLL22}&   0.3372 \chenn{($2.26 \times 10^{-6}$)}      &   0.4324 \chenn{($1.48 \times 10^{-4}$)}      &  0.1948 \chenn{($3.46 \times 10^{-7}$)}   & 0.1975 \chenn{($1.80 \times 10^{-7}$)}  \\
    \midrule
    KG4SL~\cite{wang2021kg4sl}&   0.3357 \chenn{($4.68 \times 10^{-6}$)}      &   0.4254 \chenn{($1.43 \times 10^{-4}$)}      &  0.2175 \chenn{($1.50 \times 10^{-6}$)}   & 0.2208 \chenn{($1.04 \times 10^{-6}$)}  \\
    PiLSL~\cite{liu2022pilsl}&   0.4098 \chenn{($3.71 \times 10^{-6}$)}     &   0.4035 \chenn{($3.84 \times 10^{-6}$)}     &  0.5153 \chenn{($4.95 \times 10^{-4}$)}  & 0.5149\chenn{($2.40 \times 10^{-3}$)} \\
    NSF4SL~\cite{wang2022nsf4sl}&   0.3563 \chenn{($1.34 \times 10^{-5}$)}     &   0.4626 \chenn{($2.69 \times 10^{-4}$)}     &  0.1881 \chenn{($1.85 \times 10^{-6}$)}  & 0.1818 \chenn{($3.65 \times 10^{-6}$)} \\
    KR4SL~\cite{zhang2023kr4sl}&   \underline{0.4845} \chenn{($1.50 \times 10^{-4}$)}     &   \underline{0.4901} \chenn{($4.75 \times 10^{-4}$)}     &  \underline{0.5175} \chenn{($6.82 \times 10^{-4}$)}  & \underline{0.5200} \chenn{($5.33 \times 10^{-3}$)} \\
    SLGNN~\cite{zhu2023slgnn}&  0.2172 \chenn{($1.86 \times 10^{-6}$)}     &   0.3718 \chenn{($5.13 \times 10^{-5}$)}    &  0.1259 \chenn{($1.16 \times 10^{-7}$)}  & 0.1252 \chenn{($2.34 \times 10^{-8}$)}\\
    \midrule
    DGIB4SL  &   \textbf{0.5359}    &   \textbf{0.5294}     &  \textbf{0.5485}  & \textbf{0.5484} \\
    \bottomrule
    \end{tabular}%
  \label{tab:performance}%
\end{table*}%

\subsection{Performance evaluation}
We evaluated the empirical performance of DGIB4SL against  state-of-the-art baselines, as summarized in Table~\ref{tab:performance}. 
Baseline performance was referenced from the public leaderboards provided by SLB~\cite{feng2023benchmarking}, except for KR4SL, which was based on our experimental results. As shown in Table~\ref{tab:performance}, DGIB4SL consistently outperformed all  baselines on the SynLethDB 2.0 dataset~\cite{wang2022synlethdb}. Specifically, 
KR4SL achieved the second best performance on  NDCG@50, Recall@10,  Precision@10,  Precision@50, MAP@10 and MAP@50, while PiLSL and NSF4SL achieved the second best performance on NDCG@10 and  Recall@50, respectively.
Our DGIB4SL further improved over KR4SL by 9.9$\%$, 26.7$\%$, 10.6$\%$, 8.0$\%$, 6.0$\%$ and 5.5$\%$ on NDCG@50, Recall@10,  Precision@10,  Precision@50, MAP@10 and MAP@50, respectively,  and outperformed PiLSL and NSF4SL  by 11.5$\%$ and 14.2$\%$ in  NDCG@10 and Recall@50, respectively. From these results, we draw the following conclusions: (i) the competitive performance of   DGIB4SL and KG-based baselines highlights the value of KGs in providing biological context  for gene related label prediction. 
(ii) The integration of motifs significantly enhances model performance by expanding the receptive field and encoding high-order edges into predictions effectively.

\subsection{\chenn{Explanation evaluation}}
\subsubsection{Qualitative analysis}
\begin{figure}[hbtp]
\centering
\includegraphics[width=1.01\columnwidth]{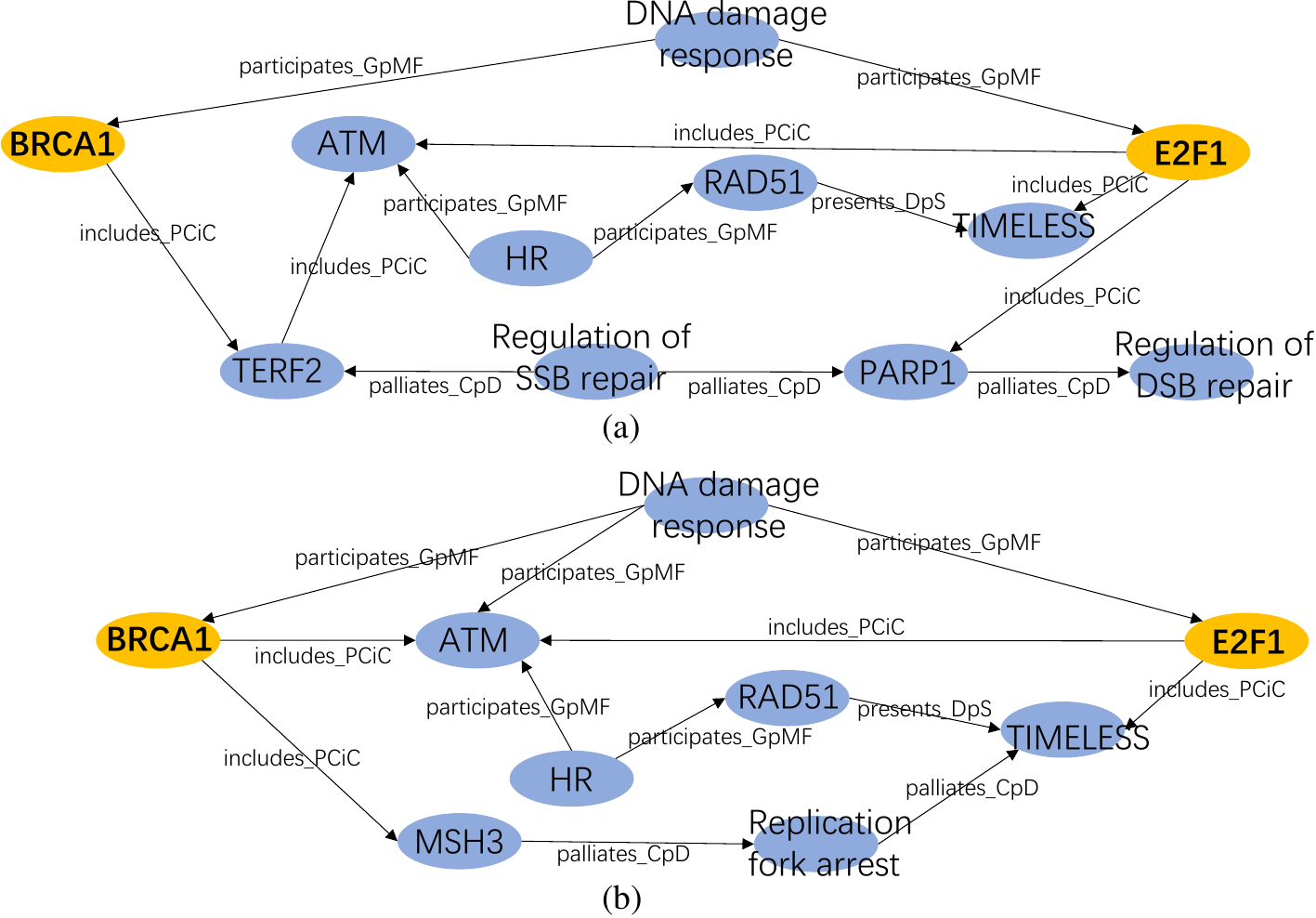}
\caption{Two explanations  learned from our DGIB4SL, provide different insights into the biological mechanisms underlying SL of the same gene pair (BRCA1, E2F1). For details on the edge nomenclature, please refer to Table S4 in Appendix.}
\label{fig:case_study}
\end{figure}

Leveraging the DGIB mechanism (Eq.~\ref{equ:gib_ddp2}), our DGIB4SL not only predicts SL interactions but also provides $K \ge 1$ \chen{explanations} that reveal the biological mechanisms underlying the  predictions for the same gene pair. For this case study, we selected the SL pair BRCA1 and \chen{E2F1} from the test data, where the predicted interaction between BRCA1 and \chen{E2F1} matched the actual label. 
To remove unimportant edges from the enclosing core graphs of (BRCA1, \chen{E2F1}), we applied edge sampling probabilities with thresholds 0.58 and 0.76 for the first and second core subgraphs distribution (Eq.~\ref{equ:AB}), respectively. Edges with probabilities  exceeding these thresholds ($(B^{uv}_1)_{i,j}$$>$$0.58$, $(B^{uv}_2)_{i,j}$$>$$0.76$) were retained. The filtered core graphs are shown in Fig.~\ref{fig:case_study} (a) and (b).

We first analyzed the first core subgraph (Fig.~\ref{fig:case_study}(a)).  The first core subgraph  highlights two key mechanisms of SL between BRCA1 and E2F1:
(1) HR Deficiency Due to BRCA1 Mutation: Pathway ``BRCA1 $\xrightarrow{\text{includes\_PCiC}}$TERF2$\xrightarrow{\text{includes\_PCiC}}$ ATM $\xleftarrow{\text{participates\_GpMF}}$ HR$\xrightarrow{\text{participates\_GpMF}}$RAD51$\xrightarrow{\text{presents\_DpS}}$ TIMELESS'' indicates that BRCA1 mutation inactivates the HR pathway. This leaves double-strand breaks (DSBs), converted from unresolved single-strand breaks (SSBs), unrepaired.
(2) SSB Repair Pathway Blockage: The pathways ``E2F1$\xrightarrow{\text{includes\_PCiC}}$ PARP1$\xrightarrow{\text{palliates\_CpD}}$ Regulation of DSB repair'' and ``Regulation of SSB repair$\xrightarrow{\text{palliates\_CpD}}$ PARP1'' demonstrate that E2F1 mutation weakens both SSB and DSB repair functions. These combined defects in SSB repair and HR result in unrepairable DNA damage, genomic instability, and ultimately cell death. Previous studies~\cite{choi2019e2f1} have shown that E2F1 depletion impairs HR, disrupting DNA replication and causing DNA damage, further supporting these findings.

We then analyzed the second core subgraph (Fig.~\ref{fig:case_study}(b)). The second core subgraph identifies a different mechanism, centered around replication fork blockage, while maintaining the shared premise of homologous recombination repair pathway loss. Specifically, the pathway ``E2F1$\xrightarrow{\text{includes\_PCiC}}$TIMELESS $\xleftarrow{\text{palliates\_CpD}}$Replication Fork Arrest'' reveals that E2F1 mutation destabilizes replication forks, leading to stalled replication. TIMELESS, a downstream target of E2F1, plays a critical role in stabilizing replication forks during DNA replication stress.
\begin{table}[htbp]
  \centering
  \caption{\chenn{Comparison of attention weights in KG-Based SL prediction methods with explanations in DGIB4SL in terms of Fidelity, Sparsity, and Diversity. Symbols $\uparrow$ and $\downarrow$ respectively represent that larger and smaller metric values are better. Values in parentheses indicate paired $t$-test $p$-values against DGIB4SL.}}
  \scriptsize
    \begin{tabular}{llll}
    \toprule
          & \chenn{Infidelity$\downarrow$}  & \chenn{Sparseness$\uparrow$} & \chenn{DPP$\uparrow$} \\
    \midrule
    \chenn{KG4SL} & \chenn{$4.0\!\times\!10^5(2\!\times\!10^{-4})$}     & \chenn{0.330($8\!\times\!10^{-4}$)}     & \chenn{-} \\
    \chenn{PiLSL} & \chenn{$5.7\!\times\!10^5(9\!\times\!10^{-3})$}     & \chenn{0.340($4\!\times\!10^{-4}$)}    & \chenn{-} \\
    \chenn{SLGNN} & \chenn{$1.8\!\times\!10^6(1\!\times\!10^{-9})$}     & \chenn{0.120($6\!\times\!10^{-5}$)}    & \chenn{\underline{1.59}($8\!\times\!10^{-4}$)} \\
    \chenn{KR4SL} & \chenn{$\underline{5.8\!\times\!10^4} (2\!\times\!10^{-3})$}    & \chenn{\underline{0.352}($7\!\times\!10^{-4}$)}    & \chenn{-} \\
    \chenn{KR4SL$^*$} & \chenn{$\underline{5.8\!\times\!10^4} (1\!\times\!10^{-3})$}    &  \chenn{0.326($5\!\times\!10^{-4}$)}   & \chenn{0.48($7\!\times\!10^{-4}$)} \\
    \midrule
    \chenn{DGIB4SL} & \chenn{$\textbf{6.2}\!\times\!\textbf{10}^\textbf{3}$}    & \chenn{\textbf{0.463}}    & \chenn{\textbf{1.67}}\\
    \bottomrule
    \end{tabular}%
  \label{tab:explanation}%
\end{table}%

\subsubsection{\chenn{Quantitative analysis}}
\chenn{We evaluated the Infidelity~\cite{yeh2019fidelity} and Sparseness~\cite{chalasani2020concise} (see Appendix for these metrics descriptions) and used the DPP to evaluate the diversity of explanations generated by DGIB4SL and other  explainable SL prediction methods, including KG4SL, PiLSL, SLGNN, and KR4SL. To compare  diversity, we introduced KR4SL$^*$, a variant of KR4SL with a multi-head attention mechanism, since the explainable baselines (except for SLGNN) generate a single explanation using similar attention mechanisms. As shown in Table~\ref{tab:explanation}, DGIB4SL outperforms other methods in terms of Infidelity, Sparseness, and DPP. We draw the following conclusions:}
\begin{itemize}
    \item \chenn{\textbf{Diversity}: Despite using multi-head attention, KR4SL$^*$ showed lower DPP values, indicating that multi-head attention alone has a limited capacity for generating diverse explanations. SLGNN's DPP performance is competitive, due to the inclusion of a distance correlation regularizer that encourages independent factor embeddings, indirectly enhancing diversity.}
    \item \chenn{\textbf{Sparsity}: Baselines employing similar attention mechanisms showed comparable Sparseness values, except for SLGNN, which directly uses learnable weights to estimate the importance of different relational features.}
    \item \chenn{\textbf{Fidelity}: The Infidelity of attention-based  methods is relatively low, possibly due to the inherent instability and high-frequency biases of attention mechanisms~\cite{DBLP:conf/acl/SerranoS19,DBLP:conf/emnlp/WiegreffeP19}.}
\end{itemize}

\subsection{Model analysis}
\subsubsection{Ablation study}

As illustrated in Fig.~\ref{fig:model}, the DGIB (Eq.~\ref{equ:gib_ddp2}), DPP constraint (third line in  Eq.~\ref{equ:gib_ddp2}) and motif-based graph encoder (Eq.~\ref{equ:motif_rst}) are key components of DGIB4SL. Based on these, we derived the following variants for the ablation study: 
(1) DGIB4SLw/oM: DGIB4SL without  motif information, to assess the impact of motifs.
(2) DGIB4SLw/oB: DGIB4SL without the DGIB objective (replacing it with an attention mechanism); 
(3) DGIB4SLw/oP: DGIB4SL without the DPP constraint (essentially reducing the objective to GIB). 
To evaluate the contributions of motifs, DGIB and DPP, we compared DGIB4SL against these variants. As shown in Fig.~\ref{fig:ablation_study}, DGIB4SL consistently outperformed DGIB4SLw/oM across all metrics, highlighting the importance of incorporating high-order structures through motifs in SL prediction.  
Secondly, the performance of DGIB4SLw/oB was comparable to DGIB4SL on all metrics. This result is expected, since DGIB4SLw/oB can still extract label-related input information via attention mechanisms, even if this information may not always faithfully reflect the model's behavior. 
Thirdly, DGIB4SLw/oP also achieved  comparable performance to DGIB4SL. This is intuitive since, without DPP constraints, DGIB4SLw/oP may find $K$ similar explanations, which could overlap with one of the $K$ different explanations found by DGIB4SL. 
\chen{To further compare their explanations, we evaluated their diversity using the DPP measure, calculated as the determinant of $\widetilde G^{uv}_1$, ..., $\widetilde G^{uv}_K$ (Eqs.~\ref{equ:ddp}-\ref{equ:equ_L}). As shown in the two rightmost columns of Fig.~\ref{fig:ablation_study}, DGIB4SL produced significantly more diverse explanations compared to DGIB4SLw/oP.}

\begin{figure}[!t]
	\centering
	\includegraphics[width=0.9\columnwidth]{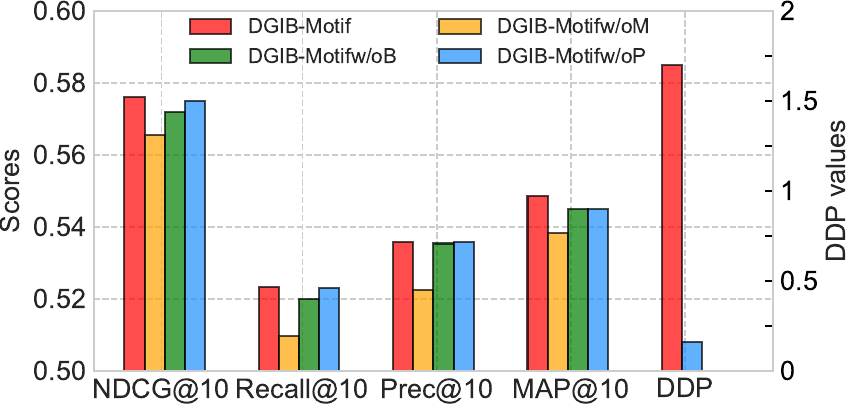} 
	\caption{\chen{Ablation study of DGIB4SL for Motif, DGIB, and DPP on NDCG@10, Recall@10, Precision@10, MAP@10 (left Y-axis) and one diversity metric DPP (right Y-axis).}}\label{fig:ablation_study}
\end{figure}

\subsubsection{Convergence analysis}
\begin{figure}[hbtp]
\centering
\includegraphics[width=1.01\columnwidth]{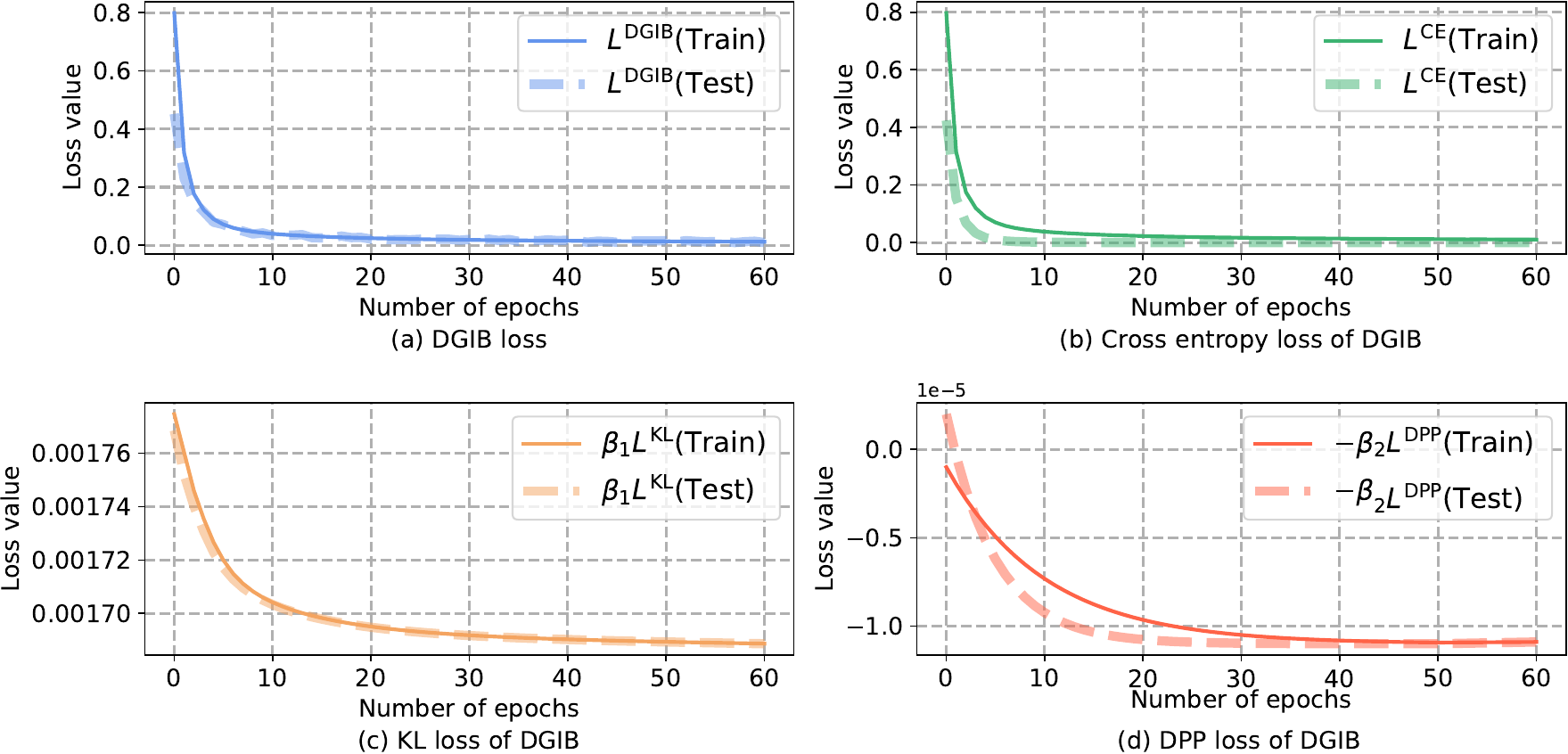}
\caption{Convergence of DGIB4SL: (a) learning curve of DGIB and (b)-(d): learning curve of each component of the loss of DGIB4SL.}
\label{fig:loss}
\end{figure}
In this section, we analyze the convergence behavior of DGIB4SL. For clarity, we rewrite the DGIB objective function in Eq.~7 as $L^{\text{DGIB}}$$=$$L^{\text{CE}} + \beta_{1} L^{\text{KL}} - \beta_{2} L^{\text{DPP}}$, where $L^{\text{CE}}$ is the binary cross-entropy loss, $\beta_{1} L^{KL}$ is the KL-divergence loss, and $-\beta_{2} L^{\text{DPP}}$ represents the DPP loss.  Figure \ref{fig:loss} illustrates the convergence trends of each component of the DGIB objective. Solid lines correspond to training set values, while dashed lines represent testing set values. As shown in Fig.~\ref{fig:loss}(a)-(b), both $L^{\text{DGIB}}$ and $L^{\text{CE}}$ experienced a steep decline during the initial epochs, with minimal separation between training and testing curves. This indicates rapid learning and effective generalization by the model. In Fig.~\ref{fig:loss}(c), $L^{\text{KL}}$ shows negligible differences between training and testing curves, suggesting that the compressed input information allows the model to generalize effectively on the test set. In contrast, Fig.~\ref{fig:loss}(d) highlights that $L^{\text{DPP}}$ initially exhibits a more pronounced gap between training and testing curves. However, this gap narrows over time, demonstrating that the model learns diverse representations effectively, albeit at a slower pace compared to other loss components.

\subsubsection{Parameter sensitivity}
\begin{figure}[h]
	\centering
	\includegraphics[width=0.95\columnwidth]{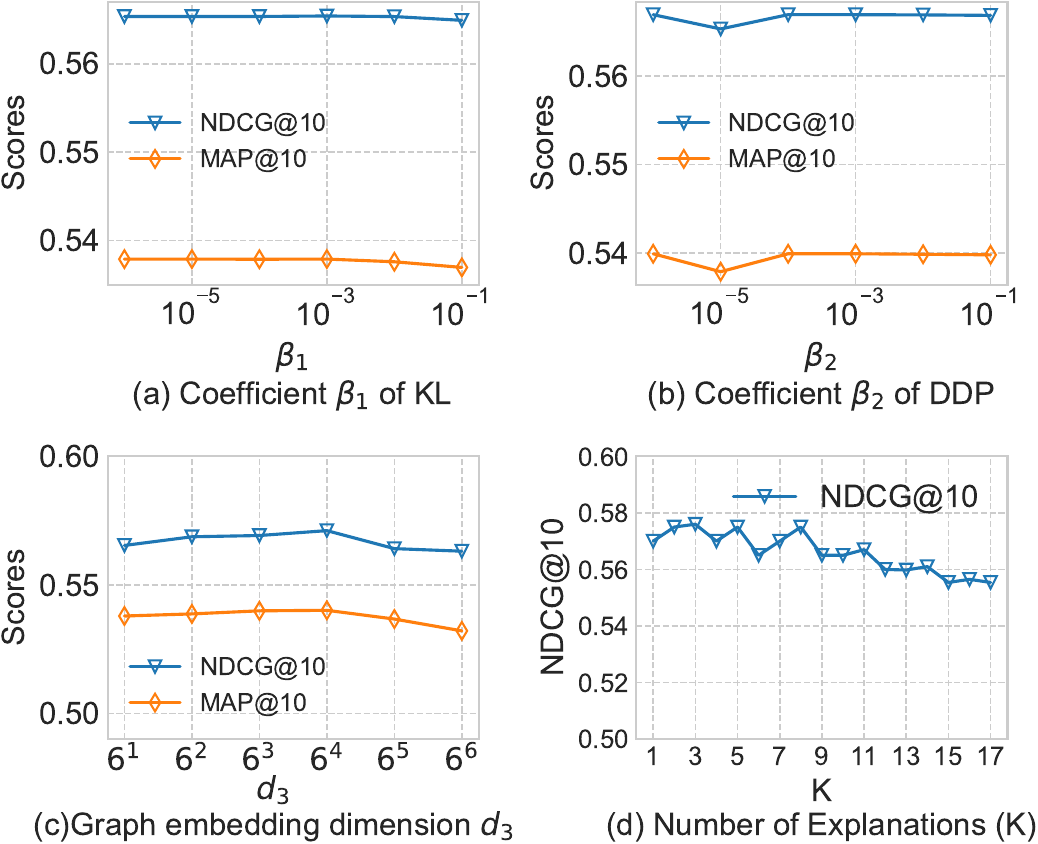} 
	\caption{Parameter sensitivity analysis for DGIB4SL on Lagrangian multipliers $\beta_1$, $\beta_2$, graph embedding dimension $d_3$ \chenn{and the number of explanations $K$ generated for each gene pair on performance}.}\label{fig:sensitivity}
\end{figure}
We explored the impact of the Lagrangian multipliers $\beta_1$ and  $\beta_2$ in Eq.~\ref{equ:AB}, the graph representation dimension $d_3$ in Eq.~\ref{equ:qz}, \chenn{and the number of explanations $K$ generated by DGIB4SL for each gene pair} on SL prediction performance. The performance trend is shown in Fig.~\ref{fig:sensitivity}. From the results, we observed the following:
(1) As illustrated in Fig.~\ref{fig:sensitivity}(a)-(b), DGIB4SL's performance  was relatively insensitive to $\beta_1$ and $\beta_2$. Specifically, $\beta_1$ values in the range $[10^{-6}, 10^{-3}]$ and $\beta_2$ values in the range $[10^{-4}, 10^{-2}]$ typically yielded robust and reliable performance. For most cases, $\beta_1$$=$$ \beta_2$$=$$10^{-4}$ proved to be an optimal choice.
(2) Fig.~\ref{fig:sensitivity}(c) indicates that increasing $d_3$ gradually improved performance, peaking around  $d_3$$=$$6^4$, but began to decline afterward, likely due to overfitting. As the performance gain beyond $d_3$$=$$6$  was modest, we opted for $d_3$$=$$6$  to simplify the model while maintaining strong performance. 
(3) \chenn{Fig.~\ref{fig:sensitivity}(d) shows the impact of $K$ within the approximate range of $[1, 17]$ (heuristically estimated; see Appendix for details). The results indicate  that DGIB4SL performs stably when $K\leq 9$, primarily due to two reasons:}
\begin{itemize}
    \item \chenn{When the actual number of core subgraphs $c_i$ in the $i$-th enclosing graph satisfies $c_i \geq K$, DGIB4SL only needs to identify at least one core subgraph to make accurate predictions, and the specific number of identified core subgraphs has minimal impact on the results.}  
    \item \chenn{When $c_i$$<$$K$, the setting of $\beta_2$$=$$10^{-4}$ in DGIB biases the trade-off toward relevance (GIB) over diversity (DPP), causing DGIB4SL to prioritize generating core subgraphs relevant to the labels, even if some explanations may overlap.}
\end{itemize}
\chenn{For $K$$>$$9$, DGIB4SL's performance declines, primarily due to the increased number of parameters in the relational edge weight module (Eq. S11), which leads to overfitting. 
We set $K$$=$$3$ for DGIB4SL because most gene pairs have no more than three core subgraphs. This choice helps effectively prevent overfitting and reduce explanation overlap.}

\subsection{Stability analysis}
\begin{figure}[!t]
	\centering
	\includegraphics[width=\columnwidth]{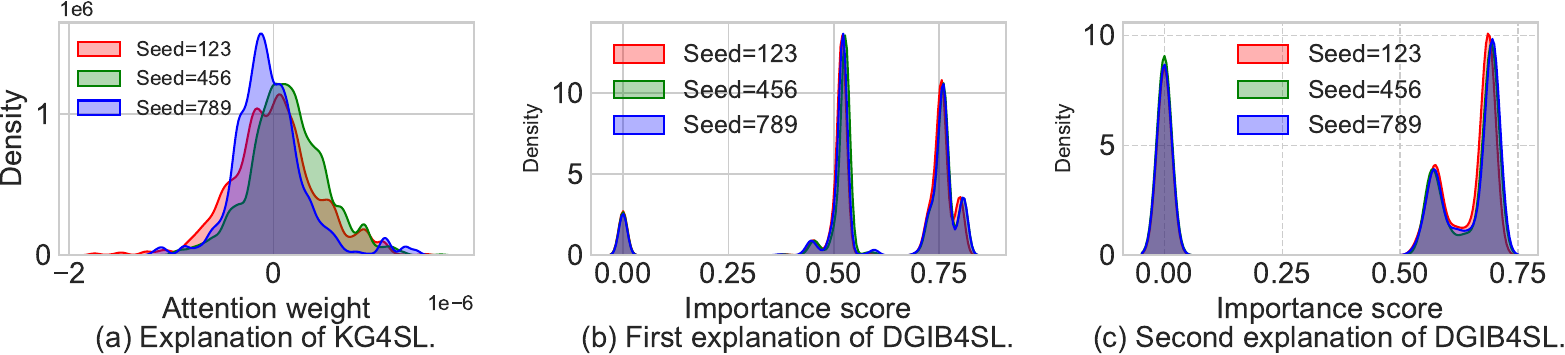} 
	\caption{Edge weight distribution of DGIB4SL and KG4SL for the same gene pair (ACTR10, PELO) under different random seeds.}\label{fig:stability_analysis}
\end{figure}

\chen{To evaluate the stability of DGIB4SL and attention-based methods, we introduced noise using three distinct random seeds to compare the edge importance distributions. Specifically, we ran DGIB4SL and KG4SL three times with different random seeds. Here, KG4SL was selected as a representative of attention-based methods due to its straightforward design and interpretability. For each run, kernel density estimation~\cite{parzen1962estimation} was applied to compute the distributions of the importance scores for each edge within the core graph of the gene pair (ACTR10, PELO).
As shown in Fig.~\ref{fig:stability_analysis}(a), the (unnormalized) attention weight distribution generated by KG4SL is unstable. In contrast, Fig.~\ref{fig:stability_analysis}(b) and (c) show that the distributions of edge weight (i.e., $(B_k^{uv})_{i,j}$ in Eq.~\ref{equ:AB}) generated by DGIB4SL for its two explanations largely overlap across different random seeds, demonstrating the stability of our DGIB4SL’s explainability.
}

\section{Conclusion and discussion}
We present Diverse Graph Information Bottleneck for Synthetic Lethality (DGIB4SL), an interpretable SL prediction framework that ensures trustworthy and diverse explanations. DGIB4SL introduces a novel DGIB objective with a Determinantal Point Process constraint to enhance diversity and employs a motif-based strategy to capture high-order graph information. A variational upper bound is proposed to address computational challenges, enabling efficient estimation. Experimental results show that DGIB4SL outperforms all baselines on the SynLethDB 2.0 dataset.

\chenn{A key limitation of DGIB4SL lies in the fixed number \( K \) for generating explanations, which may result in overlapping or incomplete explanations. Future work could explore an adaptive mechanism to dynamically adjust \( K \)  for each enclosing graph.}
\chenn{Additionally, DGIB4SL is a general framework for interaction prediction and can be applied to other domains requiring diverse and interpretable explanations, such as drug-drug interaction prediction and functional genomics research.}
\begin{framed}
\textbf{Key Points:}
\begin{itemize}
    \item We propose an interpretable knowledge graph neural network DGIB4SL that predicts synthetic lethality interactions with diverse explanations.
    \item We use the GIB principle to define a core subgraph of a gene pair, and extend the GIB objective to handle data with multiple core subgraphs, resulting in DGIB, which serves as the objective for DGIB4SL. 
    \item We apply motif-based GNNs to capture high-order graph structures.
    \item The model's effectiveness is validated through real-world data and case studies.
\end{itemize}    
\end{framed}


\bibliographystyle{unsrt}
\bibliography{reference}

\renewcommand{\thesection}{Appendix \Alph{section}}
\setcounter{section}{0} 
\section{Appendix}

\subsection{Explainability in graph neural networks}
\begin{table*}[htbp]
	\centering
	\caption{Notations and Descriptions.}
	\begin{tabular}{ll}
		\toprule
		Notations & Descriptions \\
		\midrule
		$(T)_{ij}$, $(T)_i$ & Element at the $i$-th row and the $j$-th column of matrix $T$,  and the $i$-the row of $T$.\\ 
		$G$$=$$(A, X, E)$     & Joint graph data of the SL graph and KG. \\
		$A$, $X$, $E$  &  Adjacency matrix, node feature matrix, and edge feature matrix of $G$.\\
		$G^{uv}$$=$$(A^{uv}, X^{uv}, E^{uv})$     &  Enclosing graph data for genes $u$ and $v$. \\
		$Y$ & Binary SL  interaction label of the gene pair $(u, v)$.\\
		$K \ge 1$  & Number of explanations  generated by DGIB4SL for a gene pair.\\
		$d_0, d_1, d_2, d_3$ & Dimensions of input features for nodes and relationships, relational space, and graph representation.\\
		$\widetilde G^{uv}_k$$=$$(\widetilde A_k^{uv}, \widetilde X_k^{uv})$ & $k$-th IB-graph data of $G^{uv}$.\\
		$\widetilde Z^{uv}_k \in \mathbb R^{d_3}$  & Graph representation of $\widetilde G^{uv}_k$.\\
		$f^\mu_\phi(\widetilde Z^{uv}_k)$, $f^{\Sigma}_\phi(\widetilde Z^{uv}_k)$ & Mean vector and diagonal covariance matrix of the distribution of   $\widetilde Z_k^{uv}$.\\
		$L^{uv} \in \mathbb R^{K \times K}$ & Inner product of $K$ graph representations $\widetilde Z^{uv}_1$, ..., $\widetilde Z^{uv}_K$.\\
		$B^{uv}_k$  &  Learnable edge importance weight matrix of $\widetilde G_k^{uv}$.\\
		$\beta_1, \beta_2$  &  Coefficients of KL and DPP in the DGIB objective function.\\
		$\Omega_i,  M^{uv}_i$  & Motif-based adjacency matrix of $\widetilde G^{uv}$ for a given motif $\Omega_i$.\\
		$W^1_k$, $W^2_k$, $\theta$, $\phi$ & Learnable parameters.\\
		\bottomrule
	\end{tabular}%
	\label{tab:notations}%
\end{table*}%

As GNNs are increasingly applied to SL prediction, understanding the reasoning behind their predictions becomes critical. The explainability of GNNs can be broadly categorized into two classes: post-hoc explanations and self-explainable GNNs. Post-hoc methods build an additional explainer to interpret a trained GNN, using techniques such as gradients~\cite{sundararajan2017axiomatic}, perturbation~\cite{ying2019gnnexplainer,luo2020parameterized}, or interpretable linear agents~\cite{DBLP:journals/tkde/HuangYTSC23}. 
However, these methods often fail to reveal the true reasoning process due to the inherent non-convexity and complexity of  GNNs~\cite{DBLP:conf/pkdd/PengLYZZS24,longa2024explaining}. 
Self-explainable GNNs address the limitations of  post-hoc approaches by providing predictions and explanations simultaneously. Two major directions in this area are: 
(1) Information bottleneck (IB) approaches: these methods use the IB principle~\cite{DBLP:journals/corr/physics-0004057} as a training objective to extract subgraphs closely related to graph labels~\cite{miao2022interpretable,sun2022graph,yu2022improving,seo2024interpretable}.
(2) Prototype learning: approaches like ProtGNN~\cite{zhang2022protgnn} identify subgraphs most relevant to graph patterns associated with specific classes. 
Beyond these, other methods like DIR~\cite{DBLP:conf/iclr/WuWZ0C22} identifies causal patterns via distribution interventions and models classifiers based on causal and non-causal components, while GREA~\cite{liu2022graph} introduces environment replacement to generate virtual data examples for better pattern identification. 
Despite their advantages, these methods face challenges when applied to SL prediction. Similar to KG-based approaches, self-explainable GNNs are limited in generating multiple explanations and capturing higher-order graph structures essential for prediction.

\subsection{Intuition for Measuring Diversity in Eq.~3}
Intuitively, each entry in the matrix $L^{uv}$ in Eq.~4 represents the similarity between the representations of two elements in the set, computed using dot products.  Eq.~3 measures the diversity of elements within a set by leveraging the principle that, when the feature vectors of the elements are not linearly correlated, the volume of the hypercube they form is maximized.  Consequently, the determinant of $L^{uv}$ in Eq.~3 quantifies this volume, effectively capturing the diversity of the feature vectors~\cite{kulesza2012determinantal}.

\subsection{Proof of Eq.~6}\label{sec:bound_igy}%
\begin{proposition}\label{pro:bound_iy}
	\textbf{(Upper bound of $-I(Y;\widetilde G^{uv}_k)$)}.
	For a graph $G^{uv}$ with label  $Y$ and the $k$-th IB-Graph $\widetilde G_k^{uv}$ learned from $G^{uv}$, we have
	\begin{equation}\label{equ:bound_iy}
	\small
	-I\left(Y; \widetilde G_k^{uv}\right) \le -\mathbb E_{Y, \widetilde G_k^{uv}} \Big[\log q_\theta(Y|\widetilde G_k^{uv})\Big]  +  H(Y),
	\end{equation}
	where $q_\theta(Y| \widetilde G^{uv}_k)$ parameterized by $\theta$ is the variational approximation of $p(Y| \widetilde G^{uv}_k)$.
\end{proposition}
\vspace{-2em} 
\begin{proof}
	\begin{equation}\label{equ:asfd}
	\small
	\begin{aligned}
	I(Y; \widetilde G^{uv}_k) &= \int \int p(Y, \widetilde G^{uv}_k) \log \frac{p(Y, \widetilde G^{uv}_k)}{p(Y)p(\widetilde G^{uv}_k)} dY d \widetilde G^{uv}_k\\
	&= \int \int p(Y, \widetilde G^{uv}_k) \log \frac{p(Y| \widetilde G^{uv}_k)}{p(Y)} dY d \widetilde G^{uv}_k\\
	&=\mathbb E_{Y, \widetilde G^{uv}_k} \Big[\log\frac{p(Y|\widetilde G^{uv}_k)}{p(Y)}\Big]
	\end{aligned}
	\end{equation}
	Since $p(Y| \widetilde G^{uv}_k)$ is intractable, we introduce a variational approximation $q_\theta(Y| \widetilde G^{uv}_k)$ for it. Then we have
	\begin{equation}\label{equ:sagggx}
	\small
	\begin{aligned}
	&I(Y; \widetilde G^{uv}_k)\\
	&= \mathbb E_{Y, \widetilde G^{uv}_k} \Big[\log\frac{q_\theta(Y|\widetilde G^{uv}_k)}{p(Y)}\frac{p(Y|\widetilde G^{(k)})}{q_\theta(Y|\widetilde G^{uv}_k)}\Big]\\
	&= \mathbb E_{Y, \widetilde G^{uv}_k} \Big[\log\frac{q_\theta(Y|\widetilde G^{uv}_k)}{p(Y)}\Big] +  \mathbb E_{Y, \widetilde G^{uv}_k} \Big[\log\frac{p(Y|\widetilde G^{uv}_k)}{q_\theta(Y|\widetilde G^{uv}_k)}\Big],
	\end{aligned}
	\end{equation}
	where 
	\begin{equation}\label{equ:asfgw}
	\small
	\begin{aligned}
	&\mathbb E_{Y, \widetilde G^{uv}_k} \Big[\log\frac{p(Y|\widetilde G^{uv}_k)}{q_\theta(Y|\widetilde G^{uv}_k)}\Big]\\
	&= \int\int p(\widetilde G^{uv}_k) p(Y|\widetilde G^{uv}_k) \log\frac{p(Y|\widetilde G^{uv}_k)}{q_\theta(Y|\widetilde G^{uv}_k)}   dY d \widetilde G^{uv}_k\\
	&=\mathbb E_{\widetilde G^{uv}_k} [D_{\text{KL}}(p(Y|\widetilde G^{uv}_k)||q_\theta(Y|\widetilde G^{uv}_k))].
	\end{aligned}
	\end{equation}
	Plug Eq.~\ref{equ:asfgw} into Eq.~\ref{equ:sagggx}, we have
	\begin{equation}
	\small
	\begin{aligned}
	&I(Y; \widetilde G^{uv}_k)\\
	&= \mathbb E_{Y, \widetilde G^{uv}_k} \Big[\log\frac{q_\theta(Y|\widetilde G^{uv}_k)}{p(Y)}\Big] + \mathbb E_{\widetilde G^{uv}_k} [D_{\text{KL}}(p(Y|\widetilde G^{uv}_k)||q_\theta(Y|\widetilde G^{uv}_k))]\\
	&\ge \mathbb E_{Y, \widetilde G^{uv}_k} \Big[\log\frac{q_\theta(Y|\widetilde G^{uv}_k)}{p(Y)}\Big]~(\text{non-negativity of KL Divergence})\\
	&=\mathbb E_{Y, \widetilde G^{uv}_k} \Big[\log q_\theta(Y|\widetilde G^{uv}_k)\Big] - \mathbb E_{Y, \widetilde G^{uv}_k} [\log p(Y)]\\
	&=\mathbb E_{Y, \widetilde G^{uv}_k} \Big[\log q_\theta(Y|\widetilde G^{uv}_k)\Big]    -  H(Y)~(\text{normalization of}\\
	&~~~~~~~~~~~~~~~~~~~~~~~~~~~~~~~~~~~~~~~~~~~~~~~~~\text{PDF $p(\widetilde G^{uv}_k|Y)$}\\
	&\Rightarrow  -I(Y; \widetilde G^{uv}_k) \le -\mathbb E_{Y, \widetilde G^{uv}_k} \Big[\log q_\theta(Y|\widetilde G^{uv}_k)\Big]    +  H(Y)\\
	\end{aligned}
	\end{equation}
	where $H(Y)$ is the entropy of label $Y$, which can be ignored in the optimization procedure.
\end{proof}
\begin{proposition}\label{pro:bound_gg}
	\textbf{(Upper bound of $I(G^{uv};\widetilde G^{uv}_k)$)}. For a graph $G^{uv}$ and the $k$-th IB-Graph $\widetilde G_k^{uv}$ learned from $G^{uv}$, we have
	\begin{equation}\label{equ:bound_gg}
	\small
	I(G^{uv}; \widetilde G^{uv}_k) \le \mathbb E_{G^{uv}} [D_{\text{KL}}(q_\phi(\widetilde G^{uv}_k|G^{uv})||q(\widetilde G^{uv}_k)],
	\end{equation}
	where $q(\widetilde G^{uv}_k)$$=$$\sum_{G^{uv}} p(G^{uv}) q_\phi(\widetilde G^{uv}_k | G^{uv})$ and $q_\phi(\widetilde G^{uv}_k | G^{uv})$ parameterized by $\phi$ is the variational approximation of $p(\widetilde G^{uv}_k | G^{uv})$, $D_{\text{KL}}(\cdot)$ denotes the Kullback-Leibler (KL) divergence.
\end{proposition}
\begin{proof}
	\begin{equation}\label{equ:dzg}
	\small
	\begin{aligned}
	I(G^{uv}; \widetilde G^{uv}_k) &= \int \int p(G^{uv}, \widetilde G^{uv}_k) \log \frac{p(G^{uv}, \widetilde G^{uv}_k)}{p(G^{uv})p(\widetilde G^{uv}_k)} dG^{uv} d \widetilde G^{uv}_k\\
	&= \int \int p(G^{uv}, \widetilde G^{uv}_k) \log \frac{ q_\phi(\widetilde G^{uv}_k|G^{uv})}{p(\widetilde G^{uv}_k)} dG^{uv} d \widetilde G^{uv}_k\\
	&=\mathbb E_{G^{uv}, \widetilde G^{uv}_k} \Big[\log\frac{q_\phi(\widetilde G^{uv}_k|G^{uv})}{p(\widetilde G^{uv}_k)}\Big].
	\end{aligned}
	\end{equation}
	Since $p(\widetilde G^{uv}_k)$ is intractable, we introduce a variational approximation $q(\widetilde G^{uv}_k)$$=$$\sum_{G^{uv}} p(G^{uv}) q_\phi(\widetilde G^{uv}_k | G^{uv})$ for the marginal distribution $p(\widetilde G^{uv}_k)$. Then we have
	\begin{equation}\label{equ:zs22}
	\small
	\begin{aligned}
	&I(G^{uv}; \widetilde G^{uv}_k) =\mathbb E_{G^{uv}, \widetilde G^{uv}_k} \Big[\log\frac{q(\widetilde G^{uv}_k)}{p(\widetilde G^{uv}_k)}\frac{q_\phi(\widetilde G^{uv}_k|G^{uv})}{q(\widetilde G^{uv}_k)}\Big]\\
	&=\mathbb E_{G^{uv}, \widetilde G^{uv}_k} \Big[\log\frac{q_\phi(\widetilde G^{uv}_k|G^{uv})}{q(\widetilde G^{uv}_k)}\Big] 
	+ \mathbb E_{G^{uv}, \widetilde G^{uv}_k} \Big[\log \frac{q(\widetilde G^{uv}_k)}{p(\widetilde G^{uv}_k)} \Big]\\
	&=\mathbb E_{G^{uv}, \widetilde G^{uv}_k} \Big[\log\frac{q_\phi(\widetilde G^{uv}_k|G^{uv})}{q(\widetilde G^{uv}_k)}\Big] 
	- \mathbb E_{G^{uv}, \widetilde G^{uv}_k} \Big[\log \frac{p(\widetilde G^{uv}_k)}{q(\widetilde G^{uv}_k)} \Big],
	\end{aligned}
	\end{equation}
	where
	\begin{equation}\label{equ:aftt}
	\small
	\begin{aligned}
	&\mathbb E_{G^{uv}, \widetilde G^{uv}_k} \Big[\log \frac{p(\widetilde G^{uv}_k)}{q(\widetilde G^{uv}_k)} \Big]\\
	&= \int p(\widetilde G^{uv}_k)\log \frac{p(\widetilde G^{uv}_k)}{q(\widetilde G^{uv}_k)}  \int p(G^{uv}|\widetilde G^{uv}_k) d G^{uv}  d \widetilde G^{uv}_k\\
	&= \int p(\widetilde G^{uv}_k)\log \frac{p(\widetilde G^{uv}_k)}{q(\widetilde G^{uv}_k)}   d \widetilde G^{uv}_k\\
	&= D_{\text{KL}} (p(\widetilde G^{uv}_k)||q(\widetilde G^{uv}_k))
	\end{aligned}
	\end{equation}
	Plug Eq.~\ref{equ:aftt} into Eq.~\ref{equ:zs22}, we have
	\begin{equation}
	\begin{aligned}
	\small
	&I(G^{uv}; \widetilde G^{uv}_k)\\
	&= \mathbb E_{G^{uv}, \widetilde G^{uv}_k} \Big[\log\frac{q_\phi(\widetilde G^{uv}_k|G^{uv})}{q(\widetilde G^{uv}_k)}\Big] - D_{\text{KL}} (p(\widetilde G^{uv}_k)||q(\widetilde G^{uv}_k))\\
	&\le \mathbb E_{G^{uv}, \widetilde G^{uv}_k} \Big[\log\frac{q_\phi(\widetilde G^{uv}_k|G^{uv})}{q(\widetilde G^{uv}_k)}\Big]~(\text{non-negativity of KL Divergence})\\
	&=\mathbb E_{G^{uv}} [D_{\text{KL}}(q_\phi(\widetilde G^{uv}_k|G^{uv})||q(\widetilde G^{uv}_k)]
	\end{aligned}
	\end{equation}
\end{proof}

\subsection{The calculation method of $B_k^{uv}$}
The calculation of $B_k^{uv}$ is as follows: For any edge $(i,j)$ not present in the enclosing graph $G^{uv}$, $(B_k^{uv})_{i,j}$ is fixed to 0, since our goal is to extract a subgraph from $G^{uv}$.  For edges $(i, j)$ in $G^{uv}$$=$$(A^{uv}, X^{uv}, E^{uv})$, inspired by the approaches of Wang et al.~\cite{wang2019kgat}, we optimize $(B^{uv}_k)_{i,j}$ jointly with   relational graph learning by the following function (corresponding to $\psi_k$ in Fig.~2): 
\begin{equation}\label{equ:bij}
(B_k^{uv})_{i,j} := (W^1_k (X^{uv})_j) \cdot \tanh(W^1_k (X^{uv})_i + W^2_k (E^{uv})_r)^\mathrm{T},
\end{equation}
where $W^1_k$$\in$$\mathbb R^{d_2\times d_0}$ and $W^2_k$$\in$$\mathbb R^{d_2\times d_1}$ are learnable weight matrices for mapping node and relation features into the same space, respectively. Here, $(E^{uv})_r$$\in$$\mathbb R^{d_1}$ represents the feature of relation $r$ in the triple $(i, r, j)$ and $\tanh(\cdot)$ is a nonlinear activation function. 
This formulation ensures that $(B^{uv}_k)_{i,j}$ depends on the distance between the features  $(X^{uv})_i$, $(X^{uv})_j$ of nodes $i$ and $j$ in the space defined by relation $r$~\cite{wang2019kgat}. 
Intuitively, $(B_k^{uv})_{i,j}$ reflects the SL prediction-specific importance of the edge, with smaller values indicating noise that should be assigned lower weights or excluded. Note that $\widetilde A_k^{uv}$ is not differentiable  with respect to $B_k^{uv}$ due to Bernoulli sampling, we employ the concrete relaxation~\cite{DBLP:conf/iclr/JangGP17} to approximate sampling for the IB-Graph: 
\begin{equation}\label{equ:gumbel}
(\widetilde A^{uv}_k)_{i,j} = \text{sigmoid}(\frac{1}{\tau} (\log \frac{(B^{uv}_k)_{i,j}}{1 - (B^{uv}_k)_{i,j}} + \log \frac{\epsilon}{1 - \epsilon})),
\end{equation}
where $\epsilon$$\sim$$\text{Uniform}(0,1)$ and $\tau$$\in$$\mathbb R^+$ is the temperature of the concrete distribution.

\subsection{\chenn{Construction of the IB-Graph}}
\begin{figure*}[h]
	\centering
	\includegraphics[width=1.8\columnwidth]{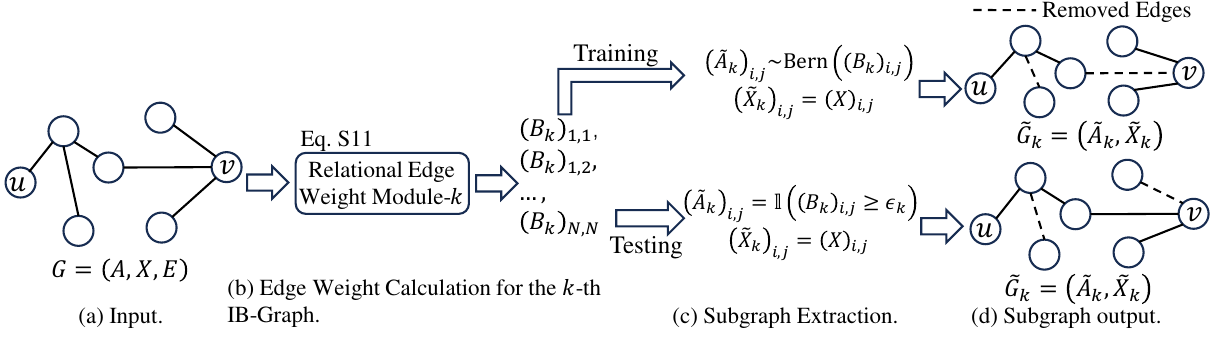} 
	\caption{\chenn{Overview for extracting the $k$-th IB-graph from the enclosing graph $G$$=$$(A,X,E)$ (for simplicity, the superscript $uv$ is omitted) around the pair of genes $u$ and $v$.}}
	\label{fig:ib_graph}
\end{figure*}
\chenn{The construction process of the IB-Graph, as shown in Fig.~\ref{fig:ib_graph}, involves four steps, which slightly differ between the training and testing phases. During the training phase:}  
\begin{enumerate}
	\item \chenn{\textbf{Input}: The enclosing graph data \( G = (A, X, E) \) (with the superscript \( uv \) omitted for simplicity) is provided for a given gene pair \( (u, v) \).}  
	\item \chenn{\textbf{Edge weight estimation}: The weights of all edges in \( G \) are estimated using Eq.~\ref{equ:bij}, which serves as the edge-weighting module.}  
	\item \chenn{\textbf{Subgraph extraction}: These edge weights are treated as parameters of independent Bernoulli variables, and random sampling assigns each edge a value of 1 or 0.}  
	\item \chenn{\textbf{Subgraph output}: The sampled edge values are aggregated to form a binary adjacency matrix \( \tilde A_k \), representing the adjacency matrix of the \( k \)-th IB-graph. The node feature matrix \( \tilde X_k \) is retained as equal to \( X \), as our focus is on edge-level explanation.} 
\end{enumerate}
\chenn{Note that sampling from a Bernoulli distribution is not differentiable, the Gumbel-Softmax trick is thus used during optimization, as detailed in Eq.~\ref{equ:gumbel}.}

\chenn{In the testing phase, the process differs only in the third step. Instead of sampling, a predefined importance threshold \( \epsilon_k \), derived from the learned importance distribution, determines whether an edge exists. In our experiments, this threshold is set as the median of the ranked edge weights in \( B_k \).} \chennn{The reason why the subgraph extraction methods in the third step differ is that subgraph selection during training is intended to optimize the model's parameters, while during testing, it aims to infer the best estimation of the explanatory subgraph.  This difference also aligns with the standard practice in GIB-based methods, as seen in the literature~\cite{miao2022interpretable,yu2022improving}.}

\subsection{Detailed information of baselines}
\begin{table}[htbp]
	\centering
	\caption{Numbers of the entities in the SynLethKG.}
	\begin{tabular}{lr}
		\toprule
		\textbf{Type} & \textbf{No. of entities} \\
		\midrule
		Gene              & 25,260 \\
		Biological process & 12,703 \\
		Side effect       & 5,702 \\
		Molecular function & 3,203 \\
		Pathway           & 2,069 \\
		Compound          & 2,065 \\
		Cellular component & 1,670 \\
		Symptom           & 427   \\
		Anatomy           & 400   \\
		Pharmacologic class & 377   \\
		Disease           & 136   \\
		\bottomrule
	\end{tabular}
	\label{tab:synlethkg_entities}
\end{table}
\begin{table}[htbp]
	\centering
	\caption{Summary of the relationship in the SynLethKG.} 
	\label{tab:relation_types}
	\begin{tabular}{lr}
		\toprule
		\textbf{Type} & \textbf{No. of entities} \\
		\midrule
		(Gene, Regulates,Gene)         & 267,791 \\
		(Gene, Interacts, Gene)         & 148,379 \\
		(Gene, Covaries, Gene)          & 62,987  \\
		(Anatomy, Expresses, Gene)      & 617,175 \\
		(Disease, Associates, Gene)     & 24,328  \\
		(Disease, Upregulates, Gene)    & 7,730   \\
		(Compound, Downregulates, Gene) & 21,526  \\
		(Disease, Downregulates, Gene)  & 7,616   \\
		(Compound, Binds, Gene)         & 16,323  \\
		(Compound, Upregulates, Gene)   & 19,200  \\
		(Anatomy, Upregulates, Gene)    & 26      \\
		(Anatomy, Downregulates, Gene)  & 31      \\
		(Gene, Participates, Cellular component)        & 97,652  \\
		(Gene, Participates, Biological process)        & 619,712 \\
		(Compound, Causes, SideEffect)  & 139,428 \\
		(Gene, Participates, Molecular function)        & 110,042 \\
		(Gene, Participates, Pathway)   & 57,441  \\
		(Compound, Treats, Disease)     & 752     \\
		(Compound, Resembles, Compound) & 6,266   \\
		(Pharmacologic Class, Includes, compound) & 1,205   \\
		(Disease, Localizes, Anatomy)   & 3,373   \\
		(Disease, Presents, Symptom)    & 3,401   \\
		(Compound, Palliates, Disease)  & 384     \\
		(Disease, Resembles, Disease)   & 404     \\
		\bottomrule
	\end{tabular}
\end{table}
\begin{table}[bp]
	\centering
	\caption{Summary of the relationship in the core graphs in Fig.~4 output by DGIB4SL. 
	}
	\label{tab:relation_types}
	\begin{tabular}{lr}
		\toprule
		\textbf{Label} & \textbf{Description} \\
		\midrule
		includes$\_$PCiC         & Pharmacological class includes compound. \\
		palliates$\_$Cpd         & Compound  palliates disease. \\
		participates$\_$GpMF          & Gene participates in molecular function.  \\
		presents$\_$Dps      & Disease presents symptom. \\
		\bottomrule
	\end{tabular}
\end{table}
We evaluated our DGIB4SL against two categories of methods: matrix factorization (MF)-based methods and graph neural network (GNN)-based methods, selecting thirteen recent approaches.
\begin{itemize}
	\item \textbf{GRSMF}~\cite{DBLP:journals/bmcbi/HuangWLOZ19} reconstructs the SL interaction graph using graph-regularized self-representative MF, incorporating PPI and GO for regularization.
	\item \textbf{SL$^2$MF}~\cite{DBLP:journals/tcbb/LiuWLLZ20} employs logistic MF for SL prediction, integrating importance weighting and PPI/GO information.
	\item \textbf{CMFW}~\cite{DBLP:journals/bioinformatics/LianyJR20} combines multiple data sources using collective matrix factorization to generate latent representations.    
\end{itemize}
GNN-based methods include:
\begin{itemize}
	\item \textbf{DDGCN}~\cite{DBLP:journals/bioinformatics/CaiCFWHW20} utilizes GCN with SL interaction matrix features and applies dropout techniques to address sparse graphs.
	\item \textbf{GCATSL}~\cite{DBLP:journals/bioinformatics/LongWLZKLL21} uses a dual attention mechanism with SL, PPI, and GO as input graphs to complete the SL graph.
	\item \textbf{SLMGAE}~\cite{DBLP:journals/titb/HaoWFWCL21} implements a multi-view graph autoencoder, integrating SL, PPI, and GO graphs for prediction.
	\item \textbf{MGE4SL}~\cite{DBLP:conf/embc/LaiCYYJWZ21} leverages a Multi-Graph Ensemble to combine PPI, GO, and Pathway data using neural network embeddings.
	\item \textbf{PTGNN}~\cite{DBLP:journals/bioinformatics/LongWLFKCLL22} pre-trains GNNs with various data sources and graph-based reconstruction features.
	\item \textbf{KG4SL}~\cite{wang2021kg4sl} is the first GNN-based model to integrate a knowledge graph for SL prediction, utilizing an attention mechanism.
	\item \textbf{PiLSL}~\cite{liu2022pilsl} extracts pairwise local subgraphs for SL prediction and integrates multi-omics data with attention mechanisms.
	\item \textbf{NSF4SL}~\cite{wang2022nsf4sl} uses contrastive learning with pre-trained KG-based features (e.g., TransE\cite{DBLP:conf/nips/BordesUGWY13}) as input for neural embeddings.
	\item \textbf{KR4SL}~\cite{zhang2023kr4sl} encodes structural information, textual semantic information, and sequential semantics for gene representations, and uses different attention mechanisms to select important edges in each hop.
	\item \textbf{SLGNN}~\cite{zhu2023slgnn} generates gene embeddings with factor-based message passing and identifies important factors through attention mechanisms, where factors consist of relationships in the KG.
\end{itemize}

\subsection{\chenn{Explainability evaluation metrics}}
\chenn{Given the lack of ground-truth explanations in the SL dataset, we employed two metrics to evaluate explanation \textbf{accuracy}:}
\begin{itemize}
	\item \chenn{\textbf{Infidelity}~\cite{yeh2019fidelity} measures explanation faithfulness, where more important features (e.g., edges or relational features) should cause larger prediction changes when altered. It introduces random perturbations to the input features, weights them by importance, and compares the changes in features and predictions to evaluate consistency.}
	\item \chenn{\textbf{Sparseness}~\cite{chalasani2020concise} evaluates explanation sparsity using the Gini index. Higher Sparseness indicates that the explanation concentrates on fewer important features, effectively reducing redundant information.}
\end{itemize}

\subsection{Implementation details}
To evaluate the effectiveness of our DGIB4SL, we adopted the same data processing, splitting, and evaluation metrics as SLB~\cite{feng2023benchmarking}. 
For dataset processing, the number of positive samples (i.e. known SL pairs) was balanced with an equal number of negative samples. Negative samples were generated using a common strategy of randomly selecting gene pairs from unknown samples.

The configuration of DGIB4SL and baseline models for the SL prediction task was as follows: For our DGIB4SL, we set $d_2$$=$$d_3$$=$$6$ (Eq.~\ref{equ:bij}, Eq. 4). Next, the coefficients $\beta_1$ and $\beta_2$ in Eq. 6 were set $\beta_1$$=$$\beta_2$$=$$10^{-4}$. 
We further set the number of explanations $K$ in Eq. 6 for each sample to 3 (\chenn{for the appropriate range of \( K \), please refer to the next section}). 
We used the Stochastic Gradient Descent (SGD) optimizer with a learning rate 0.005. The maximum number of training epochs was set as 5. For the baselines, we tuned their settings according to SLB~\cite{feng2023benchmarking}. Specifically, the output embedding dimension was set to 256 and the hidden embedding dimension was set to 512. 
The learning rate was tuned empirically. 
An early stopping strategy was used to avoid overfitting during training. The number of layers in most GNNs was set to 1, except for KR4SL, which used a 3-layer GNN to generate path-based explanations.
\chenn{We set the rank position $C$ in the ranking metrics to 10 and 50 to represent $0.1\%$ and $0.7\%$ of the 7,183 candidate gene pairs, offering two levels of ranking difficulty. These values also align with those used in related studies (e.g., NSF4SL~\cite{wang2022nsf4sl}, KR4SL~\cite{zhang2023kr4sl}).}

\subsection{\chenn{Determining the range of $K$}}
\begin{figure}[h]
	\centering
	\includegraphics[width=0.75\columnwidth]{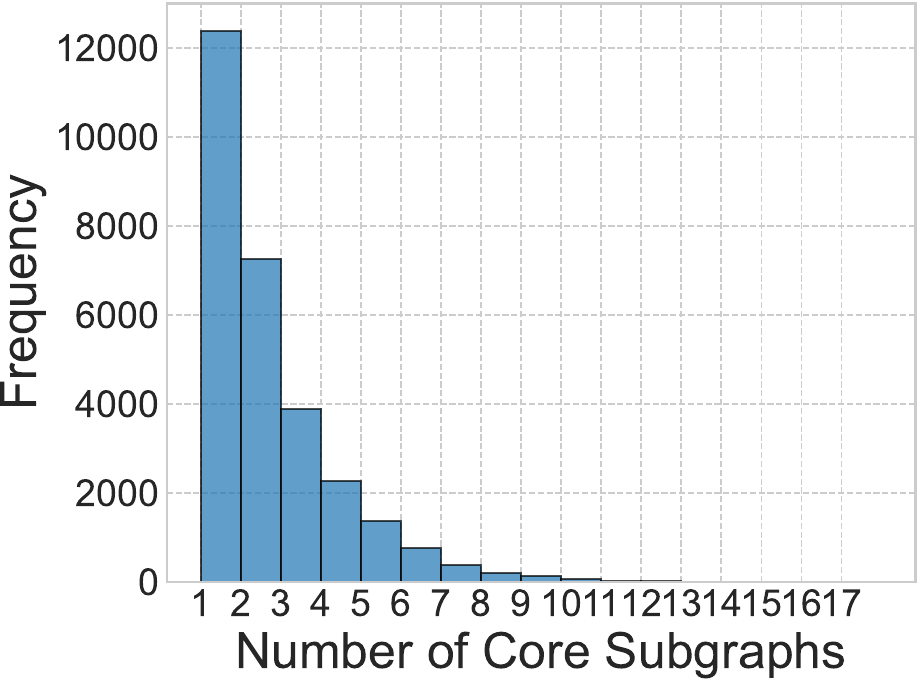} 
	\caption{\chenn{Frequency histogram of estimated core subgraphs for SL pairs enclosing graphs in the training set.}}
	\label{fig:num_core}
\end{figure}
\chenn{To determine the range of \( K \) (the number of explanations generated by DGIB4SL for each gene pair), it is necessary to clearly specify its lower and upper bounds. The lower bound is the minimum number of core subgraphs across all SL gene pairs, which is clearly 1. The upper bound represents the maximum number of core subgraphs among all possible SL gene pairs. Since calculating the exact number of core subgraphs for all SL gene pairs is unrealistic, we use a heuristic approach to estimate the range of \( K \) based on the training set. This method is guided by the intuition that core subgraphs are densely connected internally and sparsely connected externally, which is divided into 4 steps:}
\begin{enumerate}
	\item \chenn{For each SL gene pair in the training set, we compute node importance scores in the enclosing graph using the PageRank (PR) algorithm~\cite{ilprints422}.}
	\item \chenn{Nodes with importance scores below the $90$-$th$ percentile are considered unimportant and removed from the graph.}
	\item \chenn{The number of connected components formed by the remaining nodes serves as an estimate of the number of core subgraphs in the given enclosing graph.}
	\item \chenn{Fig.~\ref{fig:num_core} presents the frequency histogram of core subgraph counts in the training set. Thus, the approximate range for $K$ is [1, 17].}
\end{enumerate}

\end{document}